\documentclass{article}
\PassOptionsToPackage{dvipsnames,table}{xcolor}
\usepackage{aibl_arxiv}

\usepackage{url}

\usepackage{pifont}%

\usepackage{booktabs}
\usepackage{graphicx}
\usepackage{multirow}
\usepackage{comment}
\usepackage{subcaption}

\usepackage{bbold}

\usepackage{amsfonts,amsmath,amssymb,amsthm}
\usepackage{cleveref}
\usepackage{tocloft,titletoc}

\usepackage{fontawesome5}
\usepackage[dvipsnames]{xcolor}
\usepackage{algorithm, algorithmic}
\usepackage{wrapfig}
\usepackage{booktabs}
\usepackage{multirow}
\usepackage{colortbl}

\newtheorem{theorem}{Theorem}
\newtheorem{proposition}{Proposition}

\newtheorem{lemma}{Lemma}
\theoremstyle{definition}

\newcommand{\tr}{\operatorname{tr}}
\newcommand{\rank}{\operatorname{rank}}
\newcommand{\TV}{\mathrm{TV}}
\newcommand{\dd}{\,\mathrm{d}}
\newcommand{\norm}[1]{\lVert#1\rVert}

\newcommand{\ours}{\textsc{RCPT}}

\definecolor{oursrow}{RGB}{231,241,250}

\usepackage{amsmath,amsfonts,bm}

\def\eqref#1{equation~\ref{#1}}

\def\1{\bm{1}}

\DeclareMathAlphabet{\mathsfit}{\encodingdefault}{\sfdefault}{m}{sl}
\SetMathAlphabet{\mathsfit}{bold}{\encodingdefault}{\sfdefault}{bx}{n}

\newcommand{\E}{\mathbb{E}}

\newcommand{\Var}{\mathrm{Var}}

\newcommand{\Cov}{\mathrm{Cov}}

\AtBeginDocument{%
  \renewcommand{\eqref}[1]{\textup{(\ref{#1})}}%
}

\begin{document}

\aiblmaketitle
  {Risk-Calibrated Proposal Transport for Finite-Particle Diffusion Steering}
  {Ziseok Lee$^{1,\dagger}$, \ Jaehyeon Kim$^{1}$, \ Seungwon Kim$^{2,1}$, \ Seunghyun Moon$^{2,1,\dagger}$, \ Haneul Choi$^{2,1}$, \ Wooyeol Lee$^1$, \ Donghyun Koh$^{2,1}$, \ Minhyeong Lee$^{2,1,\dagger}$, \ Kyungsu Kim$^{1,\dagger}$}
  {$^{1}$ Seoul National University \quad $^{2}$ Asteromorph\\[0.15em]
   {\small $^{\dagger}$corresponding authors: \texttt{\{ziseoklee, minhyeong.lee, kyskim\}@snu.ac.kr, sh.moon@asteromorph.com}}}
  {Inference-time steering combines pretrained diffusion experts or rewards without retraining by changing the dynamics that transport noise to data. Feynman--Kac correction compensates for proposal mismatch through importance-weighted sequential Monte Carlo (SMC), whose finite-particle behavior depends on the proposal. Variance-controlling guidance (VCG) improves that proposal by fitting a linear drift correction to minimize empirical log-weight-rate variance. Although its population optimum cannot worsen residual variance, finite-particle VCG can nearly eliminate its fitting residual while increasing residual risk on new states by orders of magnitude. The resulting update can degrade unweighted generation or accelerate particle collapse. We show that the centered Feynman--Kac rate is the normalized transport residual and that expected out-of-fit benefit is exactly population headroom minus coefficient-estimation penalty. Under regularity assumptions, a Wasserstein analysis bounds the unweighted proposal's terminal error using this residual. These results motivate Risk-Calibrated Proposal Transport (RCPT), which uses deletion leave-one-out residuals to calibrate the retained fraction of the VCG update, adding no model calls and only small linear-algebra overhead. Experiments on 2D checker distributions, scaffold decoration, molecular property optimization, and class-conditional CIFAR-10 generation demonstrate recovery from harmful fitted updates. Across molecular and image domains, RCPT mitigates harmful fitted updates and improves a broad range of terminal metrics relative to uncalibrated VCG.}
  {
   \aiblresource{Lab}{https://aibl.snu.ac.kr}}

\section{Introduction}

Inference-time steering adapts pretrained diffusion or flow models to targets assembled from experts or rewards. These models transport noise toward data through a stochastic differential equation (SDE) \citep{song2021sde},
\[
\mathrm dX_t=v_t(X_t)\,\mathrm dt+\sigma_t\,\mathrm dW_t,
\qquad X_0\sim p_0,
\]
where $v_t$ is the drift, $\sigma_t$ is the noise amplitude, and $W_t$ is Brownian motion. An ordinary differential equation (ODE) sampler uses zero noise and an appropriate drift; steering changes the drift while keeping the pretrained models fixed.

A cheap guided proposal need not follow the requested target path. Combining expert fields, as in classifier-free guidance \citep{ho2021cfg}, gives an unweighted Anchor, but matching an instantaneous score does not ensure that the prescribed marginals satisfy the sampler's forward equation. Feynman--Kac correction (FKC) and Adaptive Path Correction with Exponents (ACE) account for this mismatch through importance weights and SMC \citep{skreta2025fkc,lee2026collapse}. Their target recovery is asymptotic: finite-particle resampling can deplete ancestral diversity, while unweighted deployment retains proposal bias.

Proposal fitting aims to improve both deployment modes by estimating a drift correction from an ensemble of samples, or particles. DriftLite's variance-controlling guidance (VCG) minimizes FK log-weight-rate variance in a linear field class, while energy-controlling guidance (ECG) fits gradient fields through an energy objective \citep{ren2026driftlite}. Population VCG cannot increase residual variance relative to Anchor because its class includes the zero correction. With finite particles, that benefit depends on whether the estimated coefficients generalize.

Limited fitting information can reverse this benefit. We focus on budgets below 100 fitting particles, with larger populations testing how calibration responds as information grows. When effective information approaches controller rank, VCG can nearly eliminate its fitting residual yet increase independent residual risk by orders of magnitude (Figure~\ref{fig:fit-deploy-main}). Increasing the fitting population costs additional expert evaluations, motivating calibration of the available update.

We treat proposal fitting as a prediction problem whose benefit must survive deployment. The centered FK rate is the normalized transport residual, and expected out-of-fit improvement equals population headroom minus coefficient-estimation penalty. This separates the improvement available to an ideal controller from its estimation cost. A Wasserstein analysis bounds unweighted terminal proposal error through the residual, which also controls leading local correction-weight inflation.

Risk-Calibrated Proposal Transport (RCPT) estimates this tradeoff and retains a scalar fraction of the VCG update. Deletion leave-one-out residuals estimate the residual scale, while a structured penalty accounts for controller complexity relative to fitting information. RCPT adds no model calls and only small linear-algebra overhead to VCG proposal fitting. It supports frozen or online deployment and retains more of VCG as fitting information grows.

Across molecular and image domains, RCPT mitigates harmful fitted updates and improves terminal metrics relative to baselines including uncalibrated VCG (Tables~\ref{tab:scaffold-simple-main}--\ref{tab:image-main}).

Our contributions are:
\begin{itemize}
\item \textbf{A finite-particle failure mode hidden by population guarantees.}
We formulate proposal fitting as an out-of-fit prediction problem and expose fit-to-deployment inversion: empirical VCG can nearly eliminate fitting residual while increasing independent risk and degrading unweighted or corrected generation. We connect the centered FK residual to a Wasserstein terminal-error bound and leading local importance-weight inflation.
\item \textbf{Deployment-risk theory and nearly cost-free calibration.}
We derive an exact headroom--penalty decomposition, an asymptotic estimation penalty, and a bias-aware scalar oracle. RCPT instantiates this principle with structured deletion-LOO calibration, no additional expert calls, and a conditional bound on its decision error.
\item \textbf{Deployment-scale evidence across inference modes and domains.}
Controlled 10,000-particle deployments and common-fit regularization audits test generalization directly. A complete 76-pocket scaffold benchmark, eight-seed molecular SMC protocols, and CIFAR-10 generation show that calibration repairs harmful VCG fits and improves downstream quality, with restored genealogy under corrected molecular sampling.
\end{itemize}

\begin{figure}[H]
\centering
\includegraphics[width=0.60\textwidth]{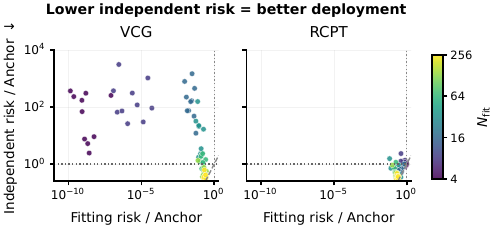}
\caption{\textbf{Low fitting risk can hide high deployment risk.} Lower on the vertical axis means better independent risk; a smaller fitting residual alone is insufficient. Both axes show log-scaled ratios to Anchor; dotted lines mark Anchor. Color denotes fitting population.}
\label{fig:fit-deploy-main}
\end{figure}

\section{Related Work}
\label{sec:positioning}

\paragraph{Diffusion steering and Anchor proposals.}
Score composition and reward guidance construct cheap proposals from pretrained models \citep{ho2021cfg,chung2025cfg,DuPMLR3Rpoe,projectdmpoe,song2023lgd,yu2023freedomsmc_,bansal2023universal,ye2024tfg}. FKC and ACE correct prescribed paths through importance weighting and SMC \citep{skreta2025fkc,lee2026collapse,doucet2000sequentialsmc,elementsofsmc}. Related particle methods support diffusion sampling, composition, and discrete correction \citep{wu2023smc,singhal2025smc,compositionsmc,reversesmc,hasan2026discretefeynmankaccorrectors,lee2025debiassmc_}. Finite-particle costs and ancestral collapse motivate improving the proposal before correction \citep{chatterjee2018sample,beskos2014stability,agapiou2017importance,koskela2020genealogies}. RCPT keeps the target fixed and calibrates a fitted update to Anchor; Table~\ref{tab:six-baseline-2d} compares both unweighted and SMC deployment.

\paragraph{Proposal fitting: VCG and ECG.}
Our closest baselines are DriftLite's \textbf{VCG}, which minimizes FK rate variance, and \textbf{ECG}, which fits gradient corrections through an energy objective \citep{ren2026driftlite}. Proposal fitting improves sampling through adaptive importance sampling, controlled or twisted SMC, and diffusion or incremental-divergence objectives \citep{cornebise2008adaptive,PerformanceBoundsforParticleFiltersUsingtheOptimalProposalsmc,heng2020controlled,whiteley2014twisted,vargas2024cmcd,phillips2024pdds,kim2025tuning}. Optimal discrete proposals, tree sampling, trust-region twists, and guidance-variance control offer complementary mechanisms \citep{ou2026optimal,jain2025tree,wang2026trust,kim2026vgas}; biased-proposal theory separates approximation and particle errors \citep{strasman2026biased}. RCPT addresses the estimation error of VCG's fitted coefficient; Table~\ref{tab:six-baseline-2d} compares VCG, ECG, and RCPT on the same Checker deployments.

\paragraph{Regression regularization and calibration.}
Regularized and selected Stein control variates address unstable regressions \citep{south2023regularized,leluc2021selection}, while classical shrinkage and covariance penalties quantify prediction error \citep{efron1973stein,efron2004covariance}. RCPT's contribution is to identify and control the deployment risk of a regression coefficient that changes generative dynamics. Ridge penalizes coefficients and principal-component regression (PCR) truncates fitted directions; RCPT instead calibrates one amplitude using structured headroom and penalty estimates. Direct LOO substitutes a cross-validated amplitude for this structured rule and serves as a calibration ablation. Table~\ref{app:tab:ray-risk} and the paired comparisons in Section~\ref{sec:discussion} show lower mean reference risk for RCPT than Direct LOO and LOO-selected ridge/PCR in the common-fit Checker audit.

\section{Finite-Particle Proposal Risk}
\label{sec:preliminaries}
\label{sec:theory}

\paragraph{Target path and Markov proposal.}
Let $\widetilde p_t>0$ be a differentiable unnormalized density,
$Z_t=\int\widetilde p_t(x)\,\dd x$, and $p_t=\widetilde p_t/Z_t$ for
$t\in[0,1]$.  In the applications, where $q^{(i)}_t$ are expert densities, $R_{j,t}$ are reward functions, and $\gamma_i(t), \beta_{j,t}$ are time-varying exponents, the target path is
\begin{equation}
\widetilde p_t(x)=
\prod_{i=1}^{I}q_t^{(i)}(x)^{\gamma_i(t)}
\exp\!\left\{\sum_{j=1}^{J}\beta_{j,t}R_{j,t}(x)\right\},
\label{eq:target-path}
\end{equation}
which includes products, ratios, time-varying exponents, and reward tilting.
Let $\mathcal L_{t,v}^{*}$ be the conservative forward Kolmogorov operator of
a proposal with selectable drift $v_t$, and let $U^v_{s,t}$ be its Markov
evolution on observables.  The unweighted proposal law $\rho_t^v$ satisfies
$\partial_t\rho_t^v=\mathcal L_{t,v}^{*}\rho_t^v$ and
$\rho_0^v=p_0$.  The corresponding raw Feynman--Kac rate is
\begin{equation}
g_t^v(x)=
\frac{\partial_t\widetilde p_t(x)-
\mathcal L_{t,v}^{*}\widetilde p_t(x)}{\widetilde p_t(x)}.
\label{eq:raw-potential}
\end{equation}
Proposal-Only deployment discards this correction; FKC/ACE retains its
incremental weights and a stated resampling rule.  RCPT changes neither the
target path nor the exponent schedule.

\paragraph{Population VCG is a projection.}
For Anchor drift $v_t^0$ and $m$ fixed fields $\phi_{j,t}$, fit an additive drift
$v_t^\lambda=v_t^0+\sum_{j=1}^m\lambda_{j,t}\phi_{j,t}$.
With unchanged diffusion, the FK rate becomes $g_t^\lambda=g_t^0+a_t^\top\lambda_t$, where
$a_{j,t}=\mathcal T_{p_t}\phi_{j,t}$ and
$\mathcal T_{p_t}\phi=p_t^{-1}\nabla\!\cdot(p_t\phi)$ is the Stein transport operator.
Define the population residual risk $V_t(\lambda)=\Var_{p_t}(g_t^\lambda)$,
$C_t=\Cov_{p_t}(a_t,a_t)$, and $d_t=\Cov_{p_t}(a_t,g_t^0)$.
The covariance-weighted seminorm
$\|z\|_{C_t}^2=z^\top C_tz=\Var_{p_t}(a_t^\top z)$ measures how much a coefficient change affects the residual. This is the \emph{$C_t$ geometry}: errors are weighted by their effect on residual variance, rather than by their Euclidean size.
The minimum-norm population VCG coefficient is
$\lambda_t^\star=-C_t^\dagger d_t$, where $\dagger$ denotes the Moore--Penrose pseudoinverse \citep{ren2026driftlite}. Writing $V_{t,\star}=V_t(\lambda_t^\star)$ gives
\begin{equation}
V_t(\lambda)=V_{t,\star}+\|\lambda-\lambda_t^\star\|_{C_t}^{2}.
\label{eq:oracle-identity}
\end{equation}
Appendix~\ref{app:proof:projection} gives the derivation, including singular $C_t$.

\subsection{Headroom and coefficient-estimation error}

We define out-of-fit risk by evaluating on the target law and averaging over fitting samples. At time $t$, sample $D$ produces $\widehat\lambda_t(D)$ with expected prediction risk $\mathcal R_t(\widehat\lambda_t)=\E_D V_t(\widehat\lambda_t(D))$. Thus $V_t$ evaluates a fixed coefficient, while $\mathcal R_t$ also averages its estimation variability. We now fix $t$ and omit time subscripts in the coefficient-risk analysis below.
\begin{theorem}[Exact out-of-fit risk decomposition]
\label{thm:out-of-fit}
For a fixed target law and field dictionary, assume
$\E_D\|\widehat\lambda\|_C^2<\infty$.
Define population headroom $H=\|\lambda^\star\|_C^2$ and coefficient-estimation penalty
$P=\E_D\|\widehat\lambda-\lambda^\star\|_C^2$. Then
\begin{equation}
\mathcal R(\widehat\lambda)-V(0)=P-H.
\label{eq:exact-out-of-fit}
\end{equation}
\end{theorem}
Taking expectation in Eq.~\eqref{eq:oracle-identity} proves the identity. The penalty includes both coefficient variance and squared bias, so it does not require unbiased fitting. The full estimated update harms expected residual risk precisely when $P>H$. Here, estimation error exceeds the available improvement.

\subsection{Scaling the fitted correction to minimize risk}

Scaling $\widehat\lambda$ by $\alpha\in[0,1]$ interpolates between Anchor and the full VCG fit. The optimal fixed amplitude balances retained signal against estimation error.
\begin{theorem}[Expected-risk oracle and the structured case]
\label{thm:shrinkage}
Let $B=(\lambda^\star)^\top C\E_D(\widehat\lambda-\lambda^\star)$ measure alignment between mean estimation error and the population coefficient. For deterministic $\alpha\in[0,1]$,
\begin{equation}
\mathcal R(\alpha\widehat\lambda)
=V_\star+(1-\alpha)^2H+\alpha^2P-2\alpha(1-\alpha)B.
\label{eq:risk-calibration}
\end{equation}
If $H+P+2B>0$, the minimizing amplitude is
\begin{equation}
\alpha^\star=\Pi_{[0,1]}\!\left(\frac{H+B}{H+P+2B}\right),
\label{eq:exact-alpha}
\end{equation}
where $\Pi_{[0,1]}$ clips to $[0,1]$; if $H+P+2B=0$, all amplitudes have the same risk. When $B=0$ and $H+P>0$, this becomes $\alpha_0=H/(H+P)$, with
\begin{equation}
\mathcal R(\alpha_0\widehat\lambda)-V(0)=-\frac{H^2}{H+P}.
\label{eq:structured-oracle-gain}
\end{equation}
\end{theorem}
The oracle retains more of the fit when headroom is large relative to estimation error and $B=0$. Unbiased fitting is sufficient for $B=0$; more generally, only its mean error's alignment with $\lambda^\star$ under $C$ must vanish. The fraction $H/(H+P)$ motivates RCPT's structured rule. Section~\ref{sec:calibration-analysis} analyzes data-selected amplitudes, and Appendix~\ref{app:sec:structured-interpretation} quantifies sensitivity to nonzero $B$.

\section{Risk-Calibrated Proposal Transport}
\label{sec:method}
\label{sec:calibration-analysis}

RCPT uses a structured estimate of prediction risk to select the strength of the fitted VCG correction.
\paragraph{Candidate fit.}
Given $\{a_n,g_n^0,\omega_n\}_{n=1}^N$ with $0\leq\omega_n<1$ and
$\sum_n\omega_n=1$, let $\bar a=\sum_n\omega_na_n$,
$\bar g^0=\sum_n\omega_ng_n^0$, and form the weighted centered moments
$\widehat C$ and $\widehat d$. The coefficient
\begin{equation}
\widehat\lambda=-(\widehat C+\eta I)^\dagger\widehat d
\label{eq:vcg-regression}
\end{equation}
minimizes $\sum_n\omega_n(g_n^0+a_n^\top\lambda-c)^2+\eta\|\lambda\|_2^2$
over $(\lambda,c)$. This is the VCG fit reused by RCPT.

\paragraph{Structured penalty and headroom.}
Delete observation $n$, renormalize $\omega_i^{(-n)}=\omega_i/(1-\omega_n)$,
and refit with the same ridge convention. With the deleted weighted means,
\begin{equation}
e_n=g_n^0+a_n^\top\widehat\lambda^{(-n)}-\widehat c^{(-n)},\qquad
\widehat c^{(-n)}=\bar g_{-n}^0+\bar a_{-n}^\top\widehat\lambda^{(-n)}.
\label{eq:deletion-loo-residual}
\end{equation}
Set $\widehat\sigma_{\rm LOO}^2=\sum_n\omega_ne_n^2$,
$\widehat N_{\rm eff}=(\sum_n\omega_n^2)^{-1}$,
$\widehat r_{\rm eff}=\tr[\widehat C(\widehat C+\eta I)^\dagger]$, and
$S=\widehat\lambda^\top\widehat C\widehat\lambda$. RCPT uses
\begin{equation}
\widehat P=\widehat\sigma_{\rm LOO}^2
\frac{\widehat r_{\rm eff}}{\widehat N_{\rm eff}},\qquad
\widehat H=[S-\widehat P]_+,\qquad
\widehat\alpha=\frac{\widehat H}{\widehat H+\widehat P+\varepsilon_{\rm num}},
\qquad\widetilde\lambda=\widehat\alpha\widehat\lambda,
\label{eq:practical-risk-calibration}
\end{equation}
where $\varepsilon_{\rm num}>0$ stabilizes a zero denominator.
The rank-over-information penalty is motivated by the classical regression
expansion $P=N^{-1}\tr(\Omega C^\dagger)+o(N^{-1})$, which reduces to
$\sigma_\epsilon^2\rank(C)/N$ when $\Omega=\sigma_\epsilon^2C$
(Lemma~\ref{lem:finite-particle}, with proof in
Appendix~\ref{app:sec:finite-proof}). The effective-rank and normalized-weight factors give a scalar approximation to this covariance penalty. They retain a low-dimensional decision when the fitting ensemble is small.

\paragraph{The criterion implemented by RCPT.}
Equation~\eqref{eq:practical-risk-calibration} is exactly
\begin{equation}
\widehat\alpha=\underset{0\leq\alpha\leq1}{\arg\min}\,
\widehat\Delta_{\rm S}(\alpha),\qquad
\widehat\Delta_{\rm S}(\alpha)
=(S+\varepsilon_{\rm num})\alpha^2-2(S-\widehat P)\alpha.
\label{eq:structured-criterion}
\end{equation}

This quadratic gives a direct operational meaning to the rule. $S$ measures the fitted correction's empirical variation, and $\widehat P$ discounts the supported improvement using deleted prediction residuals. The positive part returns Anchor when the estimated penalty exhausts the empirical gain. As supported gain grows relative to penalty, the rule retains more of the update.
\begin{algorithm}[t]
\caption{Risk-Calibrated Proposal Transport (structured $H,P$)}
\label{alg:rcpt}
\textbf{Input:} target path, Anchor, basis, weighted ensemble, ridge $\eta$.
\begin{enumerate}
\setlength{\itemsep}{0pt}
\item Fit $\widehat\lambda$ by Eq.~\eqref{eq:vcg-regression}.
\item Compute deletion residuals by Eq.~\eqref{eq:deletion-loo-residual}.
\item Form $\widehat P$, $\widehat H$, and $\widehat\alpha$ by Eq.~\eqref{eq:practical-risk-calibration}.
\item Deploy $v^0+\sum_j\widehat\alpha\widehat\lambda_j\phi_j$ online, or freeze the schedule for independent deployment.
\end{enumerate}
\end{algorithm}
\textbf{Deployment and computation.}
A frozen schedule can be reused on independent trajectories, while an online controller recalibrates at each fitted step. Proposal-Only sampling uses the selected drift directly; corrected SMC additionally retains FK weights and resampling. The implementation uses the same feature values and divergences as VCG. Additional work consists of small dense deletion-refit systems. Coefficient clipping and numerical fallbacks are specified separately in Appendix~\ref{app:sec:implementation}.

The fitted correction defines a quadratic risk along its direction. RCPT approximates this quadratic using $S$ and $\widehat P$, and their errors control the selected risk.
\begin{theorem}[Calibration error of the implemented rule]
\label{thm:structured-calibration}
Freeze a realized fit with finite residual second moments and write $f=a^\top\widehat\lambda$,
$q=\Var_p(f)$, and $\kappa=\Cov_p(g^0+f,f)$. If
$|S-q|\leq e_q$ and $|\widehat P-\kappa|\leq e_\kappa$, then RCPT satisfies
\begin{align}
V(\widehat\alpha\widehat\lambda)-V(0)
&\leq e_q+2e_\kappa,\label{eq:structured-anchor-bound}\\
V(\widehat\alpha\widehat\lambda)
-\min_{\alpha\in[0,1]}V(\alpha\widehat\lambda)
&\leq2e_q+4e_\kappa+\varepsilon_{\rm num}.
\label{eq:structured-oracle-bound}
\end{align}
\end{theorem}
For $q>0$, the ray oracle is $\alpha_{\rm ray}=\Pi_{[0,1]}(1-\kappa/q)$; when it is interior, excess risk is $q(\widehat\alpha-\alpha_{\rm ray})^2$ (Appendix~\ref{app:proof:structured}). This curvature-weighted regret measures amplitude-selection error. The Checker audit evaluates it against an independent corrected-SMC reference law (Appendix~\ref{app:sec:ray-audit}).

\subsection{Residual risk upper-bounds Wasserstein error}
\label{sec:residual-wasserstein}

The centered FK rate is the source that separates the proposal from the target path. Its minimum correction energy preserves the spatial geometry of this source and gives a direct $W_2$ bound.
\begin{theorem}[FK residual and Wasserstein transport]
\label{thm:terminal}
Under Appendix~\ref{app:sec:analytic-setup}'s conservative evolution and regularity assumptions, define $r_t^v=g_t^v-\E_{p_t}g_t^v$. Then $\partial_t\log Z_t=\E_{p_t}g_t^v$ and $\partial_tp_t=\mathcal L_{t,v}^{*}p_t+p_tr_t^v$.
Suppose additionally that the diffusion on $\mathbb R^d$ is state-independent, the laws have finite second moments, and
$\langle x-y,v_t(x)-v_t(y)\rangle\le L_t\|x-y\|^2$ with $L\in L^1(0,1)$.
Let $-\nabla\cdot(p_t\nabla\psi_t)=p_tr_t^v$ have finite energy
$\mathcal E_t(v)=\E_{p_t}\|\nabla\psi_t\|^2$, with its corrected dynamics well posed. For $\rho_0^v=p_0$,
\begin{equation}
W_2(p_1,\rho_1^v)
\le\int_0^1 K_v(t,1)\sqrt{\mathcal E_t(v)}\,\dd t,
\qquad K_v(t,1)=\exp\!\left\{\int_t^1L_s\,\dd s\right\}.
\label{eq:w2-main}
\end{equation}
If $p_t$ has Poincar\'e constant $C_P(t)$, then
$\mathcal E_t(v)\le C_P(t)V_t(v)$, where $V_t(v)=\E_{p_t}(r_t^v)^2$.
\end{theorem}
The proof in Appendix~\ref{app:sec:wasserstein} adds the minimum-energy velocity that removes the residual and couples the corrected and proposal processes with common noise. The integral retains the residual's time profile before the optional Poincar\'e conversion to variance. This specializes the transport geometry of \citet{benamou2000computational,peyre2018comparison} to fitted FK proposals.

The same residual also controls the leading local cost of FK weighting.
\begin{proposition}[Local correction-weight second moment]
\label{prop:local-condition-number}
Initialize a proposal segment at $X_t\sim p_t$. Let
$\overline W_{t,\Delta}^v=\dd\mathbb P_{t,\Delta}^v/
\dd\mathbb Q_{t,\Delta}^v$ be its exact Feynman--Kac segment weight,
normalized to mean one. If
$\overline W_{t,\Delta}^v=1+\Delta r_t^v(X_t)+o_{L^2}(\Delta)$, then
\begin{equation}
\E_{\mathbb Q}[(\overline W_{t,\Delta}^v)^2]
=1+\chi^2(\mathbb P_{t,\Delta}^v\Vert\mathbb Q_{t,\Delta}^v)
=1+\Delta^2V_t(v)+o(\Delta^2).
\label{eq:local-condition}
\end{equation}
\end{proposition}
The expansion follows by squaring the normalized local weight and taking expectation (Appendix~\ref{app:proof:local-weight}). \textbf{Residual risk therefore has two concrete roles:} it controls a transport upper bound for unweighted deployment and the leading local correction-weight variance for corrected sampling. We evaluate terminal discrepancies and genealogy separately to measure this behavior.

\section{Experiments}
\label{sec:experiments}

\textbf{Evaluation Design.}
We separate fitting from evaluation and never use terminal metrics to select a controller. Frozen schedules deploy on fresh trajectories; online controllers refit the live population. Common-state risk compares methods on shared evaluation states, whereas active-path risk follows each method's trajectory. The four experiments test independent risk, terminal generation, and particle genealogy; Appendix~\ref{app:sec:experiments} gives implementation details and cohort definitions.

\subsection{Synthetic Checker Distribution}
\label{sec:common-fit}

\begin{table}[!htb]
\caption{\textbf{Independent risk on common Checker fits.} Mean integrated Risk/Anchor over ten paired fits; lower is better. Direct LOO ablates calibration; ridge and PCR select regularization.}
\label{app:tab:ray-risk}
\centering\small
\setlength{\tabcolsep}{6pt}
\begin{tabular}{@{}crrrrrr@{}}
\toprule
$N_{\rm fit}$ & Anchor & VCG & Direct LOO & LOO ridge & LOO PCR & \bfseries\ours{}\\
\midrule
8 & 1.0000 & 244.5 & 0.9239 & 0.9211 & 1.0300 & \textbf{0.8013}\\
16 & 1.0000 & 174.0 & 1.1260 & 0.8306 & 1.1740 & \textbf{0.4931}\\
32 & 1.0000 & 2.797 & 0.3821 & 0.3705 & 0.6127 & \textbf{0.3589}\\
\bottomrule
\end{tabular}
\vspace{-1em}
\end{table}

\begin{table}[!htb]
\caption{\textbf{Checker transport.} Ten paired seeds, fit 16/deploy 2,000; all endpoints are finite. VCG/RCPT use frozen schedules and ECG the raw configuration. Lineages counts terminal ancestors among 2,000 SMC particles. Bold marks the best mean; Appendix Table~\ref{app:tab:checker-direct} ablates calibration.}
\label{tab:six-baseline-2d}
\centering\small
\setlength{\tabcolsep}{4pt}
\begin{tabular}{@{}lrrr@{\qquad}rrrr@{}}
\toprule
& \multicolumn{3}{c}{Proposal-Only} & \multicolumn{4}{c}{Corrected SMC}\\
\cmidrule(lr){2-4}\cmidrule(lr){5-8}
Controller & $W_1\downarrow$ & $W_2\downarrow$ & MMD$\downarrow$ & $W_1\downarrow$ & $W_2\downarrow$ & MMD$\downarrow$ & Lineages$\uparrow$\\
\midrule
Anchor & 0.474 & 0.646 & \textbf{0.048} & 0.367 & 0.510 & 0.048 & 338.7\\
VCG & 0.814 & 1.019 & 0.226 & 2.637 & 3.033 & 5.123 & 381.6\\
ECG & 1.394 & 1.636 & 0.325 & 0.833 & 1.113 & 0.214 & 233.7\\
\rowcolor{oursrow}\ours{} & \textbf{0.442} & \textbf{0.575} & 0.073 & \textbf{0.316} & \textbf{0.464} & \textbf{0.042} & \textbf{568.5}\\
\bottomrule
\end{tabular}
\vspace{-1em}
\end{table}

Checker isolates finite-particle fitting from independent deployment. A common-fit audit uses 12 fields, $N=8,16,32$, and ten paired fits per budget; independent corrected-SMC reference systems evaluate residual risk, and 10,000 fresh trajectories per fit test deployment. A separate matched cohort fits 16 particles and deploys 2,000 per method in both Proposal-Only and SMC modes. Appendix~\ref{app:sec:ray-audit} details the reference construction and common-fit replay, and Appendix~\ref{app:sec:experiments} specifies the matched endpoint cohort.

RCPT has the lowest mean reference risk at every fitting budget (Table~\ref{app:tab:ray-risk}) and lower mean $W_1$ and $W_2$ than Anchor, VCG, and ECG in both matched deployment modes (Table~\ref{tab:six-baseline-2d}). Under SMC, it also has the lowest MMD and retains $568.5$ terminal lineages, versus $338.7$ for Anchor. Proposal-Only MMD favors Anchor, illustrating that endpoint rankings depend on the metric. Figure~\ref{app:fig:three-regime-qualitative} visualizes generated distributions; Appendix Table~\ref{app:tab:ray-endpoints} reports common-fit endpoints and completion counts.

\begin{figure}[!htb]
\vspace{-0.5em}
\centering
\includegraphics[width=0.7\textwidth]{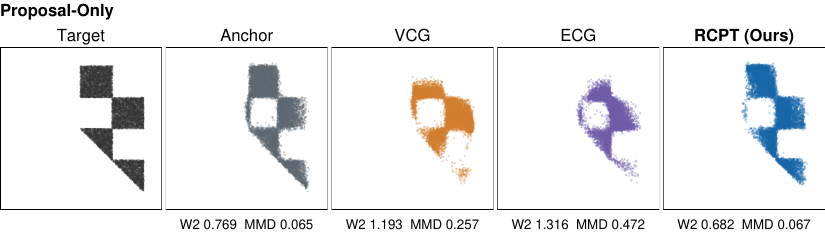}
\caption{\textbf{Separate Checker Proposal-Only replay with ECG.} Frozen fit-16 schedules,
10,000 fresh trajectories, and no resampling. Labels use a fixed
2,000-sample endpoint audit. This illustrative replay is separate from Table~\ref{tab:six-baseline-2d}. RCPT gives the lowest $W_2$; the corresponding SMC results are in Appendix Table~\ref{app:tab:checker-new}.}
\label{app:fig:three-regime-qualitative}
\vspace{-1em}
\end{figure}

\subsection{Scaffold Decoration on CrossDocked2020}

We evaluate all 76 CrossDocked2020 scaffold--pocket tasks of \citet{lee2026collapse}. Four expert paths supply fragment, ligand, conformation, and pocket information. Sixteen particles fit each 12-field schedule, which is frozen for 16 fresh paired Proposal-Only outputs. Chemistry and docking are evaluated afterward; unique useful yield counts distinct valid, scaffold-retaining molecules that beat the reference docking score per generation slot. Appendix~\ref{app:sec:scaffold-revised} specifies the models, fields, and evaluation protocol.

RCPT restores useful molecular generation after the expanded VCG fit (Table~\ref{tab:scaffold-simple-main}). Validity rises from $0.229$ to $0.434$ and unique useful yield from $0.118$ to $0.251$, with paired 95\% pocket-bootstrap improvement intervals $[0.151,0.260]$ and $[0.086,0.181]$. A separate 64-output deployment of the same schedules reproduces the VCG repair across 4,864 candidates per arm (Appendix Table~\ref{app:tab:scaffold-deploy64-confirmation}). Figure~\ref{fig:molecules-main} shows matched molecular examples.

\begin{table}[!htb]
\caption{\textbf{Scaffold decoration.} Means over 76 pockets, fit 16/deploy 16. Yields use fixed generation denominators; risk is active-path variance normalized by Anchor. Unique useful counts distinct valid, scaffold-retaining, reference-beating outputs. $^\dagger$ECG is a separate 76-pocket run.}
\label{tab:scaffold-simple-main}
\centering\small
\setlength{\tabcolsep}{4pt}
\begin{tabular}{@{}lrrrrrr@{}}
\toprule
Controller & Validity$\uparrow$ & Scaffold$\uparrow$ & Risk/Anchor$\downarrow$ &
\shortstack{Dock\\success$\uparrow$} & Affinity$\uparrow$ & \shortstack{Unique\\useful$\uparrow$}\\
\midrule
Anchor & 0.425 & 0.416 & 1.000 & 0.890 & 0.282 & 0.236\\
VCG & 0.229 & 0.225 & 26.275 & 0.569 & 0.157 & 0.118\\
ECG$^\dagger$ & 0.221 & 0.215 & 3.593 & 0.465 & 0.140 & 0.126\\
\rowcolor{oursrow}\ours{} & \textbf{0.434} & \textbf{0.419} & \textbf{0.877} & \textbf{0.903} & \textbf{0.297} & \textbf{0.251}\\
\bottomrule
\end{tabular}
\vspace{-1em}
\end{table}

\begin{figure}[!htb]
\centering
\includegraphics[width=0.7\textwidth]{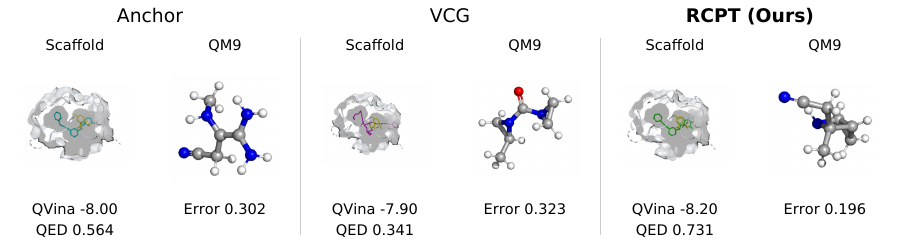}
\caption{\textbf{Molecular generation.} Matched scaffold (left: lower QVina, higher QED) and QM9 (right: lower property error) examples for each proposal. Tables~\ref{tab:scaffold-simple-main} and~\ref{tab:reward-qm9} report performance.}
\label{fig:molecules-main}
\end{figure}

\subsection{Multi-Property Molecular Reward Steering on QM9}

We steer pretrained EEGSDE \citep{bao2023equivariant} toward polarizability and HOMO--LUMO-gap targets on QM9 \citep{ramakrishnan2014quantum} using 12 gated reward-gradient fields. Each of eight split seeds fits on 16 particles and deploys eight fresh SMC batches of 16; eight online seeds instead refit 144 live particles. Independent property evaluators and canonical molecular identities measure attainment and diversity. Appendix~\ref{app:sec:reward} gives the reward, controller, and resampling details.

RCPT improves property attainment and genealogy in split deployment (Table~\ref{tab:reward-qm9}): unique successful molecules increase from $13.88$ for VCG to $23.00$ per 128 outputs, while the surviving-lineage fraction rises from $0.1611$ to $0.3525$. Online, both VCG and RCPT preserve all 144 lineages; RCPT retains a mean amplitude of $0.945$ and achieves slightly better mean evaluator error and unique success. These protocols demonstrate both repair of an unstable fit and retention of a well-supported update; Appendix Table~\ref{app:tab:reward-split-sd} reports seed variability.

\begin{table}[!htb]
\caption{\textbf{Multi-property molecular steering.} Eight-seed means. Unique success counts distinct valid target-meeting identities per 128 split or 144 online outputs. Lineages are normalized by batch size: $N=16$ split, $N=144$ online.}
\vspace{-1em}
\label{tab:reward-qm9}
\centering\small
\setlength{\tabcolsep}{5pt}
\begin{tabular}{@{}llrrr@{}}
\toprule
Protocol & Controller & Eval. MAE$\downarrow$ & Unique success$\uparrow$ & Lineages/$N\uparrow$\\
\midrule
Split $16+8\times16$ & Anchor & 0.9171 & 20.75/128 & 0.3311\\
& VCG & 0.9281 & 13.88/128 & 0.1611\\
\rowcolor{oursrow}& \ours{} & \textbf{0.8997} & \textbf{23.00/128} & \textbf{0.3525}\\
\midrule
Online $N=144$ & Anchor & 1.0570 & 0.75/144 & 0.0069\\
& VCG & 0.8027 & 45.25/144 & \textbf{1.0000}\\
\rowcolor{oursrow}& \ours{} & \textbf{0.7994} & \textbf{45.50/144} & \textbf{1.0000}\\
\bottomrule
\end{tabular}
\vspace{-1em}
\end{table}

\subsection{Improving Guided Image Generation on CIFAR-10}

Using pretrained DLSM \citep{chao2022dlsm}, we target $\widetilde p_t=q_t p_\phi(c\mid x,t)^{2.5}$ with matched Anchor guidance $\gamma=2.5$. VCG-9 and RCPT fit nine auxiliary class gradients on 16 states, then generate 320 images (32/class) with a frozen schedule and 500 probability-flow Euler steps. Independent ResNet accuracy and Inception, DINO, and CLIP metrics evaluate one seed (Appendix~\ref{app:sec:image-dlsm}).

RCPT improves all six externally evaluated metrics over Anchor and VCG-9 (Table~\ref{tab:image-main}). Relative to VCG-9, independent accuracy rises from $0.6438$ to $0.6750$, KID falls from $0.2685$ to $0.2407$, and Clean-FID falls from $89.759$ to $88.232$. A separate unit-power mechanism audit finds lower residual risk than VCG in 46/50, 47/50, and 43/50 class--seed cells at fitting populations 8, 16, and 32. Appendix Figure~\ref{app:fig:cifar-qualitative} gives selected qualitative examples.

\begin{table}[!htb]
\caption{\textbf{Class-conditional Image Generation on CIFAR-10.} Single-seed comparison using all 320 outputs per method. KID uses independent ResNet features; FID and DINO FD use Inception and DINO features, respectively. DLSM accuracy measures guider self-consistency, not independent class fidelity.}
\vspace{-1em}
\label{tab:image-main}
\centering\footnotesize
\setlength{\tabcolsep}{2.5pt}
\begin{tabular}{@{}lrrrrrrr@{}}
\toprule
Controller & Ind. acc.$\uparrow$ & DLSM acc.$\uparrow$ & KID$\downarrow$ & FID$\downarrow$ & DINO FD$\downarrow$ & DINO MMD$\downarrow$ & CLIP margin$\uparrow$\\
\midrule
Anchor & 0.6625 & 0.9625 & 0.2688 & 89.043 & 0.24740 & 0.03539 & 0.00942\\
VCG & 0.6438 & \textbf{0.9813} & 0.2685 & 89.759 & 0.24826 & 0.03381 & 0.00809\\
\rowcolor{oursrow}\ours{} & \textbf{0.6750} & 0.9625 & \textbf{0.2407} & \textbf{88.232} & \textbf{0.24534} & \textbf{0.03280} & \textbf{0.01001}\\
\bottomrule
\end{tabular}
\vspace{-1em}
\end{table}

\section{Discussion}
\label{sec:discussion}

\textbf{Do existing regularization methods already solve the problem?}
The ablations separate changes to the fitted direction through ridge or PCR from changes to the deployed amplitude through direct or structured calibration. On common Checker fits at $N=16$, RCPT's mean Risk/Anchor is $0.4931$, versus $0.8306$ for LOO-selected ridge and $1.1740$ for PCR (Table~\ref{app:tab:ray-risk}); both paired 95\% bootstrap intervals favor RCPT. RCPT captures $80.7\%$ of reference-ray improvement (Appendix Table~\ref{tab:ray-decision}); Appendix~\ref{app:sec:additional-ablation-tables} gives the paired intervals.

\textbf{Does reducing controller capacity solve the problem?}
Reducing capacity can stabilize residual fitting but sacrifice downstream chemistry. Across 76 scaffold pockets, RCPT improves validity over VCG-4 at both fitting budgets. At fit 64, affinity-qualified yield rises from $0.243$ to $0.312$ and unique useful yield from $0.206$ to $0.271$, with paired 95\% intervals excluding zero, despite VCG-4's lower residual risk (Appendix Table~\ref{app:tab:scaffold-budget}; Appendix~\ref{app:sec:scaffold-revised}). In CIFAR, RCPT has fewer harmful cells than VCG-4 (5/50 vs.\ 13/50), while VCG-4 has lower mean risk (Appendix Table~\ref{app:tab:image-matched-risk}). Although oracle directionwise shrinkage can improve on a common amplitude, estimating additional amplitudes from finite particles can overfit deletion noise and increase deployment risk, motivating RCPT's scalar calibration (Appendix~\ref{app:sec:directional}).

\textbf{The structured decision adds value beyond direct calibration.}
Against Direct LOO, RCPT improves paired Checker risk at $N=16$ and mean online QM9 unique success; its CIFAR advantages are independent accuracy, Clean-FID, and CLIP margin. On Scaffold, Direct LOO has lower common-state risk, while RCPT has stronger downstream means with paired intervals including zero (Appendix Tables~\ref{app:tab:scaffold-direct}, \ref{app:tab:qm9-calibration}, and~\ref{app:tab:image-scale15}).

\textbf{Do simple baselines already solve the problem?}
Fixed shrinkage, ridge, averaging, and norm caps can repair particular deployments, but their settings transfer unevenly. Fixed $\alpha=0.50$ exceeds RCPT's molecular yields on the ten-pocket scaffold subset, yet worsens mean QM9 evaluator error relative to VCG in the four-seed audit. Fixed ridge is similarly setting-sensitive. RCPT supplies one data-dependent rule across these tasks; the evidence supports adaptive protection, not uniform superiority over every tuned heuristic (Appendix Tables~\ref{app:tab:scaffold-simple-baselines} and~\ref{app:tab:reward-simple-baselines}).

\textbf{Supported updates remain inexpensive.}
Figure~\ref{fig:alpha-tasks} shows conservative scaffold and image amplitudes, time-varying retention on Checker, and strong retention on online QM9. In the scaffold capacity audit, increasing the fitting budget from 16 to 64 raises mean RCPT amplitude from $0.194$ to $0.906$. Calibration adds no expert calls; measured QM9 overhead is $3.13$ seconds, or $0.081\%$, with equal probed peak memory. Appendix Table~\ref{tab:cost-main} and Appendix~\ref{app:sec:cost} separate application cost from selector timings.

\begin{figure}[!htb]
\vspace{-1em}
\centering
\includegraphics[width=\textwidth]{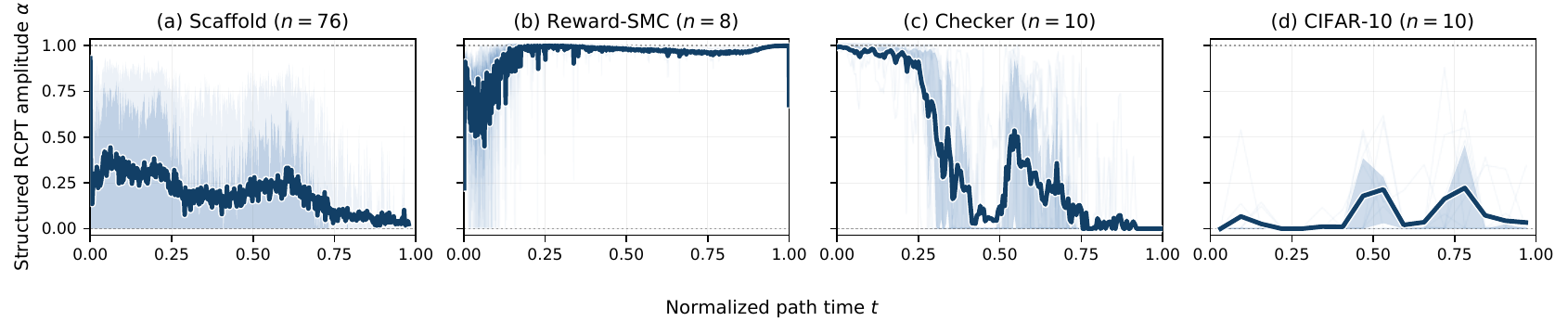}
\vspace{-2.9em}
\caption{\textbf{RCPT adapts its trust across tasks.} Mean amplitudes (thick), interquartile ranges (bands), and individual units (thin) against normalized native time. Zero recovers Anchor; one retains VCG. See Appendix~\ref{app:sec:alpha} for more details.}
\label{fig:alpha-tasks}
\end{figure}

\textbf{Why not refit VCG online?} Online refitting tracks changes in the particle distribution, but each update still estimates a controller from a finite ensemble. With $m$ particles, the centered empirical feature covariance has rank at most $m-1$. Thus, repeated fitting does not by itself increase the independent information available at a given time. RCPT applies the same calibration principle to frozen and online fits. In the online QM9 experiment with 144 particles, it retains most of the VCG update, achieves slightly better terminal performance, and preserves all lineages (Table 4). This illustrates how calibration accommodates informative online fits while controlling the amplitude supported by the available particles.

\textbf{Limitations and future directions.}
\label{sec:limitations}
The structured guarantee is conditional on calibration error, and residual upper bounds do not impose a universal ordering of endpoints or genealogy. The common-fit $N=32$ SMC audit favors Anchor over RCPT in $W_2$ by $0.0817$ (paired 95\% interval $[0.0051,0.1629]$), while the larger scaffold deployment is essentially tied with Anchor. Extending frozen-schedule guarantees to online interacting particles requires dependence and target-reference controls; lineage-aware calibration and stable grouped shrinkage are natural next steps. Multiple fitting seeds per scaffold task and broader matched regularization cohorts would further test robustness. Appendix~\ref{app:sec:reference-transfer} states the relevant transfer conditions.

\section{Conclusion}

Finite-particle proposal fitting creates a deployment decision that population control alone does not resolve. RCPT selects how much of an estimated FK correction to apply using the existing regression information. The headroom--penalty identity explains the statistical failure, and realized-ray analysis makes the quality of the selected amplitude directly testable. Common-fit comparisons show a substantial advantage over adaptive regularization at a scarce-particle budget, while controlled and molecular deployments establish practical repair of harmful VCG fits. The resulting method provides an inexpensive way to retain supported adaptation when fitting information is limited.

\subsection*{AI Use Statement}
Large language models (LLMs) were used to revise author-written text for clarity, concision, and organization, and to assist with literature retrieval, theoretical formulation, mathematical derivations and proof writing, feedback on experimental design, code development, and interpretation of results. LLMs also assisted with drafting and revising manuscript passages and preparing figures, tables, references, and LaTeX. The authors determine the final research claims and take responsibility for the correctness, originality, and verification of all AI-assisted text, proofs, code, and interpretations included in this work.

\subsection*{Ethics statement}
Generated molecules are computational candidates, not validated compounds.
Docking and learned-property scores do not establish biological activity,
synthesizability, or safety. The method inherits the limitations of the
pretrained generators and evaluators used for steering.

\subsection*{Reproducibility Statement}
Upon acceptance, we will provide supplementary source code for the structured RCPT estimator, the Direct LOO calibration ablation, and integrations for Checker, scaffold decoration, QM9 reward steering, and CIFAR-10. The package includes application drivers, configuration and dependency files, and CPU unit tests for the core estimator; its README describes setup commands, entrypoints, and external asset requirements. Appendices~\ref{app:sec:theory}--\ref{app:sec:directional} state the theoretical assumptions and proofs. Appendix~\ref{app:sec:implementation} specifies weighted fitting, deletion refits, and numerical guards, while Appendix~\ref{app:sec:experiments} documents datasets, fitting and deployment protocols, evaluation metrics, and statistical units. Pretrained weights, datasets, generated samples, and experiment outputs are not bundled. These materials support inspection and testing of the calibration rule and document the setup required for application-scale reproduction.

\section*{Acknowledgements}
This work was partly supported by the KHIDI grant funded by the Korean government (MOHW) [No.RS-2025-02307233, No.RS-2026-25613012], the NRF or IITP grants funded by the Korean government (MSIT) [No.05-26-04-0094, No.RS-2026-25472075, No.RS-2025-02305581, No.RS-2025-25442338, and No.RS-2021-II211343], the ITIP grant funded by the Korean government (MOTIR) [No.RS-2026-25549946], the Research grant from SNU, the Strategic Hub grant for International Research Collaboration of SNU, and the AI Seoul Tech Research Support Program of the Seoul Future Foundation. Kyungsu Kim is affiliated with the School of Transdisciplinary Innovations, Department of Biomedical Science, Interdisciplinary Program in AI (IPAI) \& Graduate School of AI, Medical Research Center, and AI Institute at SNU.

\bibliography{main}
\bibliographystyle{plainnat}

\clearpage
\appendix

\setcounter{secnumdepth}{2}
\numberwithin{equation}{section}
\numberwithin{theorem}{section}
\numberwithin{proposition}{section}
\numberwithin{corollary}{section}
\numberwithin{lemma}{section}
\numberwithin{table}{section}
\numberwithin{figure}{section}
\numberwithin{algorithm}{section}
\setcounter{tocdepth}{1}

{\huge \centering \bfseries Appendix \par}
\addcontentsline{toc}{section}{Appendix}

\startcontents[appendix]
\printcontents[appendix]{}{1}{}

\clearpage

\section{Supplementary Overview}

The supplement first presents additional ablations and cost measurements, then the proofs needed by the main paper, implementation details, and additional evaluations. RCPT denotes the structured $H,P$ rule in all reported tables and figures. Direct LOO is defined as a calibration ablation. Time runs from noise at $t=0$ to the terminal law at $t=1$, unless an experiment states its native reverse-time coordinate.

\paragraph{Scope of the guarantees.}
The transport bounds concern a fixed proposal schedule and the specified target path. The $H,P,B$ identities average over fitting samples, while the structured calibration theorem controls a realized decision through its quadratic approximation errors. The directional oracle and counterexample explain the choice of one amplitude. None of these results orders all terminal metrics, certifies the online interacting-particle loop, or identifies weight ESS with independent regression information.

\paragraph{Broader impact.}
RCPT may reduce the particle and memory cost of inference-time steering.
It inherits the limitations and potential misuse of the underlying generators:
generated molecules remain computational hypotheses, and image steering can
also produce misleading content. Proposal calibration supplies no scientific
validity or safety guarantee.

\section{Additional Ablation and Cost Tables}
\label{app:sec:additional-ablation-tables}

\subsection{Common-fit calibration and paired comparisons}

The paired $N=16$ comparison supports the structured decision. RCPT-minus-comparator Risk/Anchor differences are $-0.6331$ for Direct LOO (95\% CI $[-1.363,-0.0669]$), $-0.3375$ for LOO-selected ridge (95\% CI $[-0.7484,-0.0510]$), and $-0.6804$ for LOO-selected PCR (95\% CI $[-1.182,-0.2518]$). All three paired fitting-seed bootstrap intervals exclude zero, indicating lower mean reference risk at the pairwise 95\% level, conditional on the reference construction.

The reference-ray oracle quantifies the useful correction available along the fitted direction. For mean normalized risks $\bar R$, define the descriptive fraction
\[
\frac{1-\bar R(M)}{1-\bar R(\text{ray oracle})}.
\]
Anchor gives zero and the reference-ray oracle gives one. The diagnostic oracle risks are $0.5758$, $0.3717$, and $0.2570$ at $N=8,16,32$, respectively. Applying this expression to the reported aggregate risks gives RCPT approximately $46.8\%$, $80.7\%$, and $86.3\%$ at $N=8,16,32$. This is a ratio of aggregate improvements, rather than a mean of per-fit fractions. Table~\ref{tab:ray-decision} pairs it with the reported curvature-weighted ray regret. The oracle uses evaluation states and serves only as a diagnostic.

\begin{table}[!htb]
\caption{\textbf{Same-ray calibration ablation.} Captured improvement uses the diagnostic oracle risks specified above and the mean normalized risks in Table~\ref{app:tab:ray-risk}. Regret is the mean integrated excess risk over the reference-ray oracle. All deployable rows use the same full-data VCG direction.}
\label{tab:ray-decision}
\centering\small
\begin{tabular}{@{}lrr rr rr@{}}
\toprule
& \multicolumn{2}{c}{$N=8$} & \multicolumn{2}{c}{$N=16$} & \multicolumn{2}{c}{$N=32$}\\
\cmidrule(lr){2-3}\cmidrule(lr){4-5}\cmidrule(lr){6-7}
Decision & Captured & Regret & Captured & Regret & Captured & Regret\\
\midrule
Anchor & $0\%$ & 15.720 & $0\%$ & 23.290 & $0\%$ & 27.540\\
Direct LOO & $17.9\%$ & 12.900 & $-20.1\%$ & 27.960 & $83.2\%$ & 4.635\\
\rowcolor{oursrow}\bfseries\ours{} & \textbf{46.8\%} & \textbf{8.359} & \textbf{80.7\%} & \textbf{4.500} & \textbf{86.3\%} & \textbf{3.776} \\
\bottomrule
\end{tabular}
\end{table}

\subsection{Scaffold capacity and calibration}

The capacity study and matched calibration cohort answer different questions and retain separate means. Appendix~\ref{app:sec:scaffold-revised} gives the fitting and deployment protocols.

\begin{table}[!htb]
\caption{\textbf{Scaffold controller-capacity ablation.} The same 76 pockets and 16 deployment slots per arm. VCG-4 retains the ungated fields; VCG-12 and RCPT use all 12 fields. Mean RCPT amplitude is 0.194 at fit 16 and 0.906 at fit 64. All yields use the fixed generation denominator.}
\label{app:tab:scaffold-budget}
\centering\small
\setlength{\tabcolsep}{3pt}
\begin{tabular}{@{}lrrrrrrr@{}}
\toprule
Fit & Controller & Validity$\uparrow$ & Scaffold$\uparrow$ & Risk/Anchor$\downarrow$ & \shortstack{Dock\\success$\uparrow$} & Affinity$\uparrow$ & \shortstack{Unique\\useful$\uparrow$}\\
\midrule
16 & VCG-4 & 0.383 & 0.374 & 0.650 & 0.867 & 0.257 & 0.225\\
16 & VCG-12 & 0.229 & 0.225 & 26.275 & 0.569 & 0.157 & 0.118\\
\rowcolor{oursrow}16 & \ours{} & 0.434 & 0.419 & 0.877 & 0.903 & 0.297 & 0.251\\
\midrule
64 & VCG-4 & 0.354 & 0.335 & 0.297 & 0.936 & 0.243 & 0.206\\
64 & VCG-12 & 0.335 & 0.322 & 1.069 & 0.915 & 0.236 & 0.200\\
\rowcolor{oursrow}64 & \ours{} & 0.441 & 0.422 & 0.548 & 0.911 & 0.312 & 0.271\\
\bottomrule
\end{tabular}
\end{table}

\begin{table}[!htb]
\caption{\textbf{Matched Scaffold calibration ablation.} A separate fit-16/deploy-16 cohort on all 76 pockets. Anchor, Direct LOO, and RCPT isolate the addition of calibration and its structured form; the full VCG fit is shown for reference. These means are kept separate from the earlier cohort. Risk is the active-path diagnostic; paired uncertainty is reported in Appendix~\ref{app:sec:scaffold-revised}.}
\label{app:tab:scaffold-direct}
\centering\small
\setlength{\tabcolsep}{3.4pt}
\begin{tabular}{@{}lrrrrrr@{}}
\toprule
Controller & Validity$\uparrow$ & Scaffold$\uparrow$ & Risk/Anchor$\downarrow$ & \shortstack{Dock\\success$\uparrow$} & Affinity$\uparrow$ & \shortstack{Unique\\useful$\uparrow$}\\
\midrule
Anchor & 0.438 & 0.415 & 1.000 & 0.898 & 0.294 & 0.251\\
VCG-12 & 0.220 & 0.213 & 21.184 & 0.562 & 0.143 & 0.111\\

Direct LOO & 0.418 & 0.400 & 0.756 & 0.914 & 0.289 & 0.247\\
\rowcolor{oursrow}\ours{} & 0.438 & 0.423 & 0.876 & 0.906 & 0.303 & 0.262\\
\bottomrule
\end{tabular}
\end{table}

\subsection{QM9 and image calibration}

These ablations use the online QM9 budget and the CIFAR terminal-generation cohort described in Appendices~\ref{app:sec:reward} and~\ref{app:sec:image-dlsm}.

\begin{table}[!htb]
\caption{\textbf{QM9 calibration ablation.} Eight online seeds, $N=144$.
Anchor has no fitted correction; Direct LOO uses direct deleted-loss
calibration; RCPT uses the structured headroom--penalty decision.
Unique success counts distinct valid target-meeting identities per 144 outputs.}
\label{app:tab:qm9-calibration}
\centering\small
\setlength{\tabcolsep}{6pt}
\begin{tabular}{@{}lrrr@{}}
\toprule
Calibration & MAE$\downarrow$ & Unique success /144$\uparrow$ & Lineages/$N\uparrow$\\
\midrule
Anchor & 1.0570 & 0.75 & 0.0069\\
Direct LOO & 0.8019 & 43.13 & 0.9479\\
\rowcolor{oursrow}\ours{} & \textbf{0.7994} & \textbf{45.50} & \textbf{1.0000}\\
\bottomrule
\end{tabular}
\end{table}

\begin{table}[!htb]
\caption{\textbf{CIFAR-10 calibration ablation.} The same 320-output seed and nine auxiliary fields as Table~\ref{tab:image-main}. Direct LOO and RCPT calibrate the same fitted direction. Bold marks the best value in each column.}
\label{app:tab:image-scale15}
\centering\footnotesize
\setlength{\tabcolsep}{2.5pt}
\begin{tabular}{@{}lrrrrrrr@{}}
\toprule
Controller & Ind. acc.$\uparrow$ & DLSM acc.$\uparrow$ & KID$\downarrow$ & FID$\downarrow$ & DINO FD$\downarrow$ & DINO MMD$\downarrow$ & CLIP margin$\uparrow$\\
\midrule
Anchor & 0.6625 & 0.9625 & 0.2688 & 89.043 & 0.24740 & 0.03539 & 0.00942\\
Direct LOO & 0.6563 & \textbf{0.9656} & \textbf{0.2308} & 88.722 & \textbf{0.24523} & \textbf{0.03241} & 0.00979\\
\rowcolor{oursrow}\ours{} & \textbf{0.6750} & 0.9625 & 0.2407 & \textbf{88.232} & 0.24534 & 0.03280 & \textbf{0.01001}\\
\bottomrule
\end{tabular}
\end{table}

\begin{table}[!htb]
\caption{\textbf{CIFAR-10 target-matched risk and capacity audit.} Fifty class--seed cells, five seeds, fit 16, and $\beta=\gamma=2.5$. Each independent-state risk is divided by its cell's matched Anchor risk before averaging. These are common-state surrogate risks. Harm means a ratio above one; bold marks the best fitted method.}
\label{app:tab:image-matched-risk}
\centering\small
\setlength{\tabcolsep}{8pt}
\begin{tabular}{@{}lrrr@{}}
\toprule
Controller & Mean Risk/Anchor$\downarrow$ & Harm$\downarrow$ & Worst ratio$\downarrow$\\
\midrule
Anchor & 1.0000 & 0/50 & 1.0000\\
VCG-9 & 0.9968 & 22/50 & 1.0246\\
VCG-4 & \textbf{0.9869} & 13/50 & 1.0401\\
Direct LOO & 0.9970 & 13/50 & 1.0204\\
\rowcolor{oursrow}\ours{} & 0.9977 & \textbf{5/50} & \textbf{1.0126}\\
\bottomrule
\end{tabular}
\end{table}

\subsection{Application cost}

In the common-fit Checker implementation, mean selector times per 200-step schedule are $0.132$ seconds for RCPT, $1.31$ for Direct LOO, $9.34$ for selected ridge, and $13.03$ for selected PCR. These implementation-specific selector timings are separate from the hardware-identified generation measurements below. Appendix~\ref{app:sec:cost} gives the instrumentation.

\begin{table}[!htb]
\caption{\textbf{Measured application cost.} Scaffold means use all 76 pockets on an RTX A6000 (48 GB). QM9 uses four recorded A100-SXM4 (80 GB) seeds; peak memory is from a matched 10-step probe. Times exclude offline evaluation. Full instrumentation is in Appendix~\ref{app:sec:cost}.}
\label{tab:cost-main}
\centering\small
\begin{tabular}{@{}llrr@{}}
\toprule
Setting & Controller & Generation time (s) & Peak memory (GB)\\
\midrule
Scaffold, 16 particles & VCG & 201.715 & 5.070\\
\rowcolor{oursrow} & \ours{} & 197.457 & 4.742\\
\midrule
QM9, 144 particles & VCG & 3879.518 & 40.204\\
\rowcolor{oursrow} & \ours{} & 3882.645 & 40.204\\
\bottomrule
\end{tabular}
\end{table}

\section{Transport Residual and Correction Weights}
\label{app:sec:theory}

\subsection{Analytic setup}
\label{app:sec:analytic-setup}

Let $\widetilde p_t>0$ have finite normalizer $Z_t$ and sufficient space-time
regularity to differentiate under the integral. For each fixed admissible
drift schedule, assume the conservative operator $\mathcal L_{t,v}^*$
generates a Markov evolution $U_{s,t}^{v,*}$, the forward equation has a unique
mild solution, and $\int_0^1\|p_tr_t^v\|_{L^1}\,\dd t<\infty$.
Boundary fluxes vanish. Assume $r_t^v\in L^2(p_t)$ for the stated risk bounds.
Random schedules fitted independently of deployment can be conditioned on
throughout. These assumptions do not identify a particle-dependent online
drift with a fixed single-particle Markov process.

Nonsmooth or vanishing terminal densities require a limiting argument.
One may apply the dynamic result up to $1-\tau$ on a regular positive path,
then let $\tau\downarrow0$ if the endpoint laws converge in $W_2$ and the
energy bound remains uniformly integrable. The triangle inequality and
this convergence pass the upper bound to the limit.

\subsection{Proof of Theorem~\ref{thm:terminal}}
\label{app:proof:terminal}

\begin{proof}
The definition of $g_t^v$ gives
$\partial_t\widetilde p_t=\mathcal L_{t,v}^*\widetilde p_t+g_t^v\widetilde p_t$.
Integrating and using conservativity yields
$\dot Z_t/Z_t=\E_{p_t}g_t^v$. Substituting $\widetilde p_t=Z_tp_t$ proves
\begin{equation}
\partial_tp_t=\mathcal L_{t,v}^*p_t+p_tr_t^v,
\qquad r_t^v=g_t^v-\E_{p_t}g_t^v.
\label{app:eq:normalized}
\end{equation}
Since $p_0=\rho_0^v$, variation of constants gives
\begin{equation}
p_1-\rho_1^v=\int_0^1U_{t,1}^{v,*}(p_tr_t^v)\,\dd t.
\label{app:eq:duhamel}
\end{equation}
Markov kernels contract the total-variation norm of signed measures. Therefore,
with $d_{\TV}(p,q)=\frac12\|p-q\|_1$,
\[
d_{\TV}(p_1,\rho_1^v)
\leq\frac12\int_0^1\E_{p_t}|r_t^v|\,\dd t
\leq\frac12\int_0^1\sqrt{V_t(v)}\,\dd t.
\]
This proves the normalized source identity. The Wasserstein part follows in Appendix~\ref{app:sec:wasserstein}. For completeness, the optional final time relaxation is
\begin{equation}
d_{\TV}(p_1,\rho_1^v)\le\tfrac12\int_0^1\sqrt{V_t(v)}\dd t
\le\tfrac12\sqrt{\int_0^1V_t(v)\dd t}.
\label{eq:tv-bound}
\end{equation}
\end{proof}

For a statistic $m\in[a,b]$, the same argument implies
$|\E_{p_1}m-\E_{\rho_1^v}m|\leq(b-a)d_{\TV}(p_1,\rho_1^v)$.
Zero residual along the entire path implies exact unweighted marginals by
uniqueness. A reduced residual bound alone does not imply exactness or a
strict reduction in terminal discrepancy: the propagated signed defects in
Eq.~\eqref{app:eq:duhamel} can cancel.

\subsection{Proof of the Wasserstein upper bound}
\label{app:sec:wasserstein}

Transport energy retains spatial information that residual variance discards.
For a zero-mass signed density $h$, define
\begin{equation}
\|h\|_{-1,p}=\sup_{\int\|\nabla f\|^2p\le1}\left|\int fh\right|,
\qquad \mathcal A_p=-\nabla\cdot(p\nabla).
\label{app:eq:negative-norm}
\end{equation}
We use weak solutions modulo constants, with zero boundary flux or sufficient
decay at infinity. If $\mathcal A_p\psi=h$ has finite energy, testing its
weak equation and using Cauchy--Schwarz gives
$\|h\|_{-1,p}^2=\int\|\nabla\psi\|^2p=\langle h,\mathcal A_p^{-1}h\rangle$.
In particular, $\mathcal E_t(v)=\|p_tr_t^v\|_{-1,p_t}^2$.

\begin{lemma}[Minimum correction energy]
Among velocities satisfying $-\nabla\cdot(p_tw)=p_tr_t^v$, the gradient
$w=\nabla\psi_t$ minimizes $\E_{p_t}\|w\|^2$. If $p_t$ has Poincar\'e
constant $C_P(t)$, then $\mathcal E_t(v)\le C_P(t)V_t(v)$.
\end{lemma}
\begin{proof}
Every admissible $w$ differs from $\nabla\psi_t$ by a field $z$ with
$\nabla\cdot(p_tz)=0$. Integration by parts gives
$\E_{p_t}\nabla\psi_t\cdot z=0$, proving minimality by the Pythagorean
identity. Choose the mean-zero representative of $\psi_t$. Then
\[
\mathcal E_t(v)=\E_{p_t}(r_t^v\psi_t)
\le\sqrt{V_t(v)C_P(t)\mathcal E_t(v)}.
\]
Division proves the claim when the energy is positive; the zero case is immediate.
\end{proof}

\paragraph{Proof of the Wasserstein bound in Theorem~\ref{thm:terminal}.}
The corrected drift $v_t+\nabla\psi_t$ has marginals $p_t$ by
Eq.~\eqref{app:eq:normalized} and uniqueness. Couple its process $X_t$ with
the proposal $Y_t$ using the same additive Brownian noise and an optimal
initial coupling. The noise cancels. With
$D_t=(\E\|X_t-Y_t\|^2)^{1/2}$, the one-sided Lipschitz condition implies
\[
\tfrac12\partial_tD_t^2
\le L_tD_t^2+D_t\sqrt{\mathcal E_t(v)}.
\]
Regularization at $D_t=0$ and Gronwall give the slightly more general bound
\begin{equation}
W_2(p_1,\rho_1^v)\le K_v(0,1)W_2(p_0,\rho_0^v)
+\int_0^1K_v(t,1)\sqrt{\mathcal E_t(v)}\dd t.
\label{app:eq:w2-bound}
\end{equation}
This proves Eq.~\eqref{eq:w2-main}. Poincar\'e is needed only to replace
energy by variance. The energy bound itself does not require it.


\paragraph{What correction changes.}
Exact FK weighting targets $p_t$ despite proposal mismatch. The preceding
terminal bounds concern the unweighted proposal law $\rho_t^v$.
They do not bound the error of a finite corrected particle system.
The local weight result below explains a separate benefit of reducing
residual variance. Particle error additionally depends on weights,
dependence, resampling, and mixing \citep{delmoral2011concentration}.

\subsection{Proof of Proposition~\ref{prop:local-condition-number}}
\label{app:proof:local-weight}

\begin{proof}
Write $\overline W-1=\Delta r_t^v(X_t)+R_\Delta$ with
$\|R_\Delta\|_2=o(\Delta)$. Then
\[
\E(\overline W-1)^2
=\Delta^2\E_{p_t}(r_t^v)^2
+2\Delta\E[r_t^v(X_t)R_\Delta]+\E R_\Delta^2
=\Delta^2V_t(v)+o(\Delta^2).
\]
Because $\E\overline W=1$, this equals both
$\E\overline W^2-1$ and $\chi^2(\mathbb P\Vert\mathbb Q)$.
\end{proof}

One sufficient route to the assumed expansion is
$A_\Delta=\int_t^{t+\Delta}g_s^v(X_s)\dd s
=\Delta g_t^v(X_t)+o_{L^2}(\Delta)$ together with
$\|e^{A_\Delta}-1-A_\Delta\|_2=o(\Delta)$.
Normalizing $e^{A_\Delta}$ then subtracts $\E_{p_t}g_t^v$ at first order.
For example, bounded rates with local $L^2$ continuity satisfy these
conditions. Uniform conclusions require uniform remainder control.

\subsection{Why local variance does not order genealogy}
\label{app:sec:weight-curse}

Let $Z$ be uniform on $\{-1,1\}$ and $0<\epsilon<1$. Two weighting schemes
use $G_1=1+\epsilon Z$ and respectively $G_2=1-\epsilon Z$ or
$G_2=1+\epsilon Z$. Every increment has mean one and squared coefficient of
variation $\epsilon^2$ under this common reference law. However,
\[
\mathrm{CV}^2(G_1G_2)=0\quad\hbox{or}\quad
\frac{4\epsilon^2}{(1+\epsilon^2)^2},
\]
respectively. Thus marginal incremental second moments do not determine
cumulative weight variation. Resampling adds further dependence on timing.
For exactly Gaussian log weight $L$ with variance $s^2$,
\begin{equation}
\frac{\E e^{2L}}{(\E e^L)^2}=e^{s^2},\qquad
\mathrm{ESS}_{\rm pop}/N=e^{-s^2},
\label{app:eq:lognormal-ess}
\end{equation}
but the variance of a sum of log increments includes their cross-covariances.

For multinomial resampling with weights $w_i$, the probability that all
children share one parent is $\sum_iw_i^N$. Resampling resets numerical
weights without restoring ancestors. We therefore report both the number
of surviving initial ancestors and lineage ESS:
\begin{equation}
N_{\rm anc}=|\{a:\pi_a>0\}|,\qquad
\mathrm{ESS}_{\rm anc}=(\sum_a\pi_a^2)^{-1},
\label{app:eq:ancestry-ess}
\end{equation}
where $\pi_a$ is the fraction of descendants of ancestor $a$.

\subsection{Time discretization and approximate residuals}
\label{app:sec:numerical-error}

On a grid $0=t_0<\cdots<t_K=1$, hold $\lambda_k$ fixed on each interval.
If $|\partial_tV_t(\lambda_k)|\leq L_k$, then
\begin{equation}
\left|\int_{t_k}^{t_{k+1}}V_t(\lambda_k)\dd t
-\Delta_kV_{t_k}(\lambda_k)\right|\leq\tfrac12 L_k\Delta_k^2.
\label{app:eq:quadrature}
\end{equation}
Consequently the continuous-time proposal with that schedule satisfies
\begin{equation}
d_{\TV}(p_1,\rho_1^{\lambda_{0:K-1}})
\leq\tfrac12\sqrt{\sum_k\Delta_kV_{t_k}(\lambda_k)
+\tfrac12\sum_kL_k\Delta_k^2}.
\label{app:eq:grid-tv}
\end{equation}
The discretized sampler has an additional error
$d_{\TV}(\rho_1^\lambda,\rho_{1,h}^\lambda)$; the bound does not assert a
universal rate for it. If $r_t=\widehat r_t+\delta_t$, then
\begin{equation}
\|r\|_{L^2(\dd t\,p_t)}
\leq\|\widehat r\|_{L^2(\dd t\,p_t)}
+\|\delta\|_{L^2(\dd t\,p_t)}.
\label{app:eq:model-error}
\end{equation}
Score, divergence, and tracked-density errors therefore remain separate from
coefficient-estimation error.

\paragraph{Wasserstein model and numerical errors.}
The direct transport argument also separates approximation errors without
converting them to TV. If the true centered source is
$r_t=\widehat r_t+\delta r_t$, then
\[
\|p_tr_t\|_{-1,p_t}
\le\|p_t\widehat r_t\|_{-1,p_t}+\|p_t\delta r_t\|_{-1,p_t}
\le\sqrt{\widehat{\mathcal E}_t}+\sqrt{C_P(t)}\|\delta r_t\|_{L^2(p_t)}.
\]
The last step alone uses Poincar\'e. A drift error $\delta v_t$ changes the
source by a weighted divergence, whose negative norm is at most
$\|\delta v_t\|_{L^2(p_t)}$ by the dual definition. Apply
Eq.~\eqref{app:eq:w2-bound} to the actual drift with its corresponding
stability factor. These are deterministic model-error terms, not extra
coefficient-estimation penalties.

For a piecewise schedule, let
$F(t)=K_v(t,1)\sqrt{\mathcal E_t(v)}$ and suppose
$|F(t)-F(t_k)|\le\omega_k(t-t_k)$ on its $k$th interval. Then
\[
W_2(p_1,\rho_1^v)\le\sum_k\Delta_k F(t_k)
+\sum_k\int_0^{\Delta_k}\omega_k(s)\dd s.
\]
This follows by integrating the modulus bound and preserves the energy
profile until the final quadrature. A discretized law $\rho_{1,h}^v$ and
an empirical output $\widehat\rho_{1,h}^v$ add, by the triangle inequality,
$W_2(\rho_1^v,\rho_{1,h}^v)$ and
$W_2(\rho_{1,h}^v,\widehat\rho_{1,h}^v)$.
Finite reference samples add their own error. For SMC, this last term is
an interacting-particle error and cannot be replaced by an iid sampling
rate without justification.

\section{Finite-Particle Risk and Structured Calibration}
\label{app:sec:finite-particles}
\label{app:sec:vcg}

\subsection{Population projection}
\label{app:proof:projection}

The linear residual formula follows from
$\mathcal L_{v+u}^*p=\mathcal L_v^*p-\nabla\!\cdot(pu)$.
Expanding its variance gives
$V(\lambda)=V(0)+2d^\top\lambda+\lambda^\top C\lambda$.
For any $z\in\ker C$, $z^\top(a-\E a)=0$ almost surely, so $z^\top d=0$.
Thus $d\in\operatorname{range}(C)$ and completing the square proves
Eq.~\eqref{eq:oracle-identity}, including singular $C$.
This is the population VCG projection of \citet{ren2026driftlite}.

\subsection{Proof of Theorem~\ref{thm:out-of-fit}}
\label{app:proof:out-of-fit}

\begin{proof}
Apply Eq.~\eqref{eq:oracle-identity} to each realized fit and take expectation:
$\mathcal R(\widehat\lambda)=V_\star+P$. At Anchor,
$V(0)=V_\star+H$. Subtraction proves the identity.
\end{proof}

\subsection{Proof of Theorem~\ref{thm:shrinkage}}
\label{app:proof:shrinkage}

\begin{proof}
Set $e=\widehat\lambda-\lambda^\star$. Then
\[
\mathcal R(\alpha\widehat\lambda)-V_\star
=\E\|-(1-\alpha)\lambda^\star+\alpha e\|_C^2
=(1-\alpha)^2H+\alpha^2P-2\alpha(1-\alpha)B.
\]
The coefficient of $\alpha^2$ is
$D_0=H+P+2B=\E\|\widehat\lambda\|_C^2\geq0$.
For $D_0>0$, minimizing the convex quadratic gives
Eq.~\eqref{eq:exact-alpha}; clipping enforces the interval constraint.
When $B=0$, substitution of $H/(H+P)$ gives
Eq.~\eqref{eq:structured-oracle-gain}. If $D_0=0$, the fitted correction is
constant $p$-almost surely for almost every fit, so its amplitude has no
effect on variance.
\end{proof}

\subsection{Proof of Theorem~\ref{thm:structured-calibration}}
\label{app:proof:structured}

\begin{proof}
Since $S,\widehat P\geq0$ and $\varepsilon_{\rm num}>0$, the minimizer of
Eq.~\eqref{eq:structured-criterion} is
$\Pi_{[0,1]}((S-\widehat P)/(S+\varepsilon_{\rm num}))$.
It equals Eq.~\eqref{eq:practical-risk-calibration}, including $S\leq\widehat P$.
For the realized field, write
\begin{equation}
\Delta(\alpha)=V(\alpha\widehat\lambda)-V(0)
=q\alpha^2-2(q-\kappa)\alpha.
\label{app:eq:realized-ray-risk}
\end{equation}
For $q>0$, the true minimizer is
$\alpha_{\rm ray}=\Pi_{[0,1]}(1-\kappa/q)$; for $q=0$, the risk is constant.
Let $e=e_q+2e_\kappa$. Since $|\alpha^2-2\alpha|\leq1$ on $[0,1]$,
\[
-e\leq\widehat\Delta_{\rm S}(\alpha)-\Delta(\alpha)
-\varepsilon_{\rm num}\alpha^2\leq e.
\]
Optimality and $\widehat\Delta_{\rm S}(0)=0$ imply
$\Delta(\widehat\alpha)\leq\widehat\Delta_{\rm S}(\widehat\alpha)+e\leq e$.
Comparing instead to $\alpha_{\rm ray}$ gives
\[
\Delta(\widehat\alpha)
\leq\widehat\Delta_{\rm S}(\alpha_{\rm ray})+e
\leq\Delta(\alpha_{\rm ray})+2e+\varepsilon_{\rm num}.
\]
Finally, for a fixed dictionary, $d=-C\lambda^\star$ gives
$\kappa=\widehat\lambda^\top C(\widehat\lambda-\lambda^\star)$.
Taking expectation yields $\E_D\kappa=P+B$.
\end{proof}

If the estimation-error event has probability at least $1-\delta$, the same
conclusions hold with that probability. Establishing such an event requires
control of the sampling law, fitted field, and residual estimator. The
algebraic bound does not supply this control automatically.

\subsection{What the structured model assumes}
\label{app:sec:structured-interpretation}

\paragraph{Bias sensitivity of the deterministic oracle.}
Write $\alpha_0=H/(H+P)$ for $H+P>0$. For nonzero $B$, let
$\alpha_B$ be the constrained oracle in Eq.~\eqref{eq:exact-alpha}. Then
\begin{equation}
0\leq\mathcal R(\alpha_0\widehat\lambda)
-\mathcal R(\alpha_B\widehat\lambda)\leq |B|/2.
\label{app:eq:bias-sensitivity}
\end{equation}
Indeed, the omitted term $-2B\alpha(1-\alpha)$ has range width $|B|/2$,
and $\alpha_0$ minimizes the remaining quadratic. When the bias-aware oracle
is interior and $D_0=H+P+2B>0$, the exact regret is
\[
\frac{B^2(P-H)^2}{(H+P)^2D_0},\qquad
\frac{H+B}{D_0}-\frac{H}{H+P}
=\frac{B(P-H)}{(H+P)D_0}.
\]
These statements concern population constants and deterministic amplitudes.
They do not transfer unchanged to an estimated $\widehat\alpha$.
In particular, $\E\|\widehat\lambda\|_C^2=H+P+2B$, so subtracting a
penalty estimate does not universally give an unbiased estimate of $H$.

\paragraph{An exact normal-means interpretation.}
In whitened coordinates, suppose
\[
Z\mid\theta\sim\mathcal N(\theta,\sigma^2I_r),\qquad
\theta\sim\mathcal N(0,\tau^2I_r).
\]
Completing the square gives
$\E[\theta\mid Z]=\tau^2(\tau^2+\sigma^2)^{-1}Z$.
This posterior mean minimizes posterior squared loss. With
$H_{\rm prior}=r\tau^2$ and $P=r\sigma^2$, its amplitude is exactly
$H_{\rm prior}/(H_{\rm prior}+P)$.
For known $\sigma^2$, maximizing the marginal normal likelihood gives
$\widehat\tau^2=[\|Z\|^2/r-\sigma^2]_+$, whose plug-in amplitude is
$[1-r\sigma^2/\|Z\|^2]_+$ (zero when $Z=0$).
This is the RCPT algebra with $S=\|Z\|^2$, known penalty $P$, and
$\varepsilon_{\rm num}=0$. It is a classical empirical-Bayes shrinkage
interpretation \citep{efron1973stein}, not a claim that the fitted diffusion
controllers follow this model. Estimating the hyperparameters, geometry, and
penalty does not inherit exact posterior optimality.

\paragraph{Why Direct LOO need not improve the decision.}
For an interior realized oracle,
$\Delta(\alpha)-\Delta(\alpha_{\rm ray})=q(\alpha-\alpha_{\rm ray})^2$.
A lower-variance estimate of the amplitude can therefore be preferable even
if its criterion is biased. At a boundary, the exact regret also includes
$2(q\alpha_{\rm ray}-(q-\kappa))(\alpha-\alpha_{\rm ray})\geq0$.
The criterion-estimation tradeoff is consistent with established comparisons
of covariance penalties and cross-validation \citep{efron2004covariance}.
It is not a universal dominance result for structured RCPT. The common-fit audit in Appendix~\ref{app:sec:ray-audit} illustrates the distinction: smaller unweighted amplitude error need not give smaller curvature-weighted ray regret. Accurate quadratic coefficients are sufficient, but not necessary, for a good decision; multiplying the entire criterion by a positive constant leaves its minimizer unchanged. Likewise, a zero-mean criterion correction at each fixed $\alpha$ need not have zero mean after minimizing over $\alpha$.

\subsection{Finite-particle penalty and its proof}
\label{app:sec:finite-proof}
\label{app:thm:finite}

\begin{lemma}[Finite-particle estimation penalty]
\label{lem:finite-particle}
Let $X_1,\ldots,X_N\overset{\mathrm{iid}}\sim p$ with a fixed dictionary,
$A=a-\E_pa$, $\epsilon=g^0-\E_pg^0+A^\top\lambda^\star$, and
$\Omega=\E_p[\epsilon^2AA^\top]$. Suppose $C=\E_pAA^\top$ has fixed rank,
its nonzero eigenvalues are positive, $\E\|A\|^4<\infty$,
$\E\epsilon^2<\infty$, $\E[\epsilon^2\|A\|^2]<\infty$, and empirical rank
is eventually equal to population rank in probability. Use minimum-norm
least squares, or ridge $\eta_N=o(N^{-1/2})$ in the normalized objective.
On $\operatorname{range}(C)$,
\begin{equation}
\sqrt N(\widehat\lambda-\lambda^\star)
\xrightarrow{d}\mathcal N(0,C^\dagger\Omega C^\dagger).
\label{eq:finite-particle}
\end{equation}
If $\{N\|\widehat\lambda-\lambda^\star\|_C^2\}$ is also uniformly integrable,
then $P=N^{-1}\tr(\Omega C^\dagger)+o(N^{-1})$.
\end{lemma}
\begin{proof}
The population normal equation is $\E A\epsilon=0$. The centered empirical
normal equations and consistency of the covariance inverse on its range give
\begin{equation}
\widehat\lambda-\lambda^\star
=-C^\dagger\frac1N\sum_nA_n\epsilon_n+o_p(N^{-1/2}).
\label{app:eq:ols-linearization}
\end{equation}
Sample centering contributes $O_p(N^{-1})$; the stipulated ridge contributes
$o_p(N^{-1/2})$. Apply the multivariate CLT. The quadratic continuous-mapping
theorem and uniform integrability then yield
$NP\to\tr(CC^\dagger\Omega C^\dagger)=\tr(\Omega C^\dagger)$.
\end{proof}

The reduction to $\sigma_\epsilon^2\rank(C)/N$ requires
$\Omega=\sigma_\epsilon^2C$. RCPT further substitutes a deleted prediction
residual for $\sigma_\epsilon^2$, effective rank for rank, and weight ESS for
sample count. These choices define its low-particle model; the asymptotic
lemma does not assert that they are accurate near interpolation.

\paragraph{Ridge and dependent particles.}
Fixed population ridge targets $\lambda_\eta=-(C+\eta I)^{-1}d$ and incurs
\begin{equation}
V(\lambda_\eta)-V(\lambda^\star)
=\eta^2(\lambda^\star)^\top(C+\eta I)^{-1}
C(C+\eta I)^{-1}\lambda^\star.
\label{app:eq:ridge-bias}
\end{equation}
Thus the preceding centered CLT is not a fixed-ridge theorem.
For independent observations with deterministic normalized weights, analogous
weighted normal equations give leading variance
$(\sum_n\omega_n^2)\tr(\Omega C^\dagger)$ under weighted consistency,
a weighted CLT, and uniform integrability. Random importance weights and
interacting particles require their own covariance analysis.
If a particle-system CLT for the regression score has covariance $\Gamma_t$,
the same linearization, covariance consistency, and uniform integrability
give $N^{-1}\tr(\Gamma_tC^\dagger)$ instead. Conventional weight ESS does
not estimate this covariance in general; equal weights after resampling do
not imply independent observations.

\paragraph{Frozen-schedule consequence.}
\phantomsection\label{app:thm:end-to-end}
For a schedule fitted independently of deployment, combine the projection
identity with the dynamic Wasserstein bound. If $\widetilde\lambda_k$ is
held on $[t_k,t_{k+1})$ and the assumptions of
Theorem~\ref{thm:terminal} hold, then
\begin{equation}
W_2(p_1,\rho_1^{\widetilde\lambda})
\le\sum_k\int_{t_k}^{t_{k+1}}K_v(t,1)
\sqrt{C_P(t)\{V_{t,\star}
+\|\widetilde\lambda_k-\lambda_t^\star\|_{C_t}^2\}}\dd t.
\label{app:eq:end-to-end}
\end{equation}
This retains the time profile of the estimation error. Theorem~\ref{thm:structured-calibration} controls the selected coefficient at each grid point under its calibration assumptions. Interpolation, model error,
and numerical integration require the separate controls above. A smaller
upper bound alone does not order endpoint errors. A guarantee for the
same-particle online loop additionally requires analysis of that interacting process.

\section{Calibration Ablation and ECG Baseline}
\label{app:sec:calibration-baselines}

\subsection{Direct LOO}
\label{app:sec:direct-loo}

Using the deleted means and fits in the main text, set
$u_n=g_n^0-\bar g_{-n}^0$ and
$z_n=(a_n-\bar a_{-n})^\top\widehat\lambda^{(-n)}$. Define
$A_D=\sum_n\omega_nu_n^2$, $b_D=\sum_n\omega_nu_nz_n$, and
$c_D=\sum_n\omega_nz_n^2$. Then
\begin{equation}
\widehat L_D(\alpha)=A_D+2b_D\alpha+c_D\alpha^2,\qquad
\widehat\alpha_D=
\begin{cases}\Pi_{[0,1]}(-b_D/c_D),&c_D>0,\\0,&c_D=0.\end{cases}
\label{app:eq:direct-loo}
\end{equation}
The formula follows by minimizing a scalar quadratic. The intercept must be
$\bar g^0_{-n}+\alpha\bar a_{-n}^\top\widehat\lambda^{(-n)}$ at each
amplitude; holding the unscaled intercept fixed changes the objective.
Both selectors deploy the full-data coefficient, not an average of deleted
coefficients. Deletion weights are renormalized and the same ridge
normalization is used in every fold.

Direct LOO estimates deleted prediction loss without a scalar covariance model. It changes only the calibration criterion, while retaining the full-data VCG direction. Appendix~\ref{app:sec:reference-transfer} distinguishes target-law risk from the reference-law diagnostics and numerical guards used in deployment.

\subsection{Residual fitting versus energy fitting}
\label{app:sec:ecg}

VCG minimizes $\Var_p(g^0+\mathcal T_pu)$ in a vector-field class.
For gradient controls $u=\nabla\psi$, energy-controlling guidance (ECG)
instead uses the weak Poisson objective \citep{ren2026driftlite}
\begin{equation}
\min_\psi\;\tfrac12\E_p\|\nabla\psi\|^2-\E_p[r^0\psi].
\label{app:eq:ecg-energy}
\end{equation}
Its first variation is
$\E_p\nabla\psi\cdot\nabla h=\E_p r^0h$, equivalent under integration by
parts to $-\nabla\!\cdot(p\nabla\psi)=pr^0$.
For $\psi=\sum_j\theta_j\psi_j$, ECG solves $M\theta=b$, with
$M_{ij}=\E_p\nabla\psi_i\cdot\nabla\psi_j$ and
$b_i=\Cov_p(\psi_i,g^0)$. The two finite-dimensional objectives need not
select the same field.

ECG provides a distinct proposal-fitting baseline. Its scalar-potential
class does not automatically contain arbitrary gated vector fields used by
VCG and RCPT. Conversely, ECG's energy objective avoids squared Stein
features in its normal matrix. These are representation and objective
differences, not a general superiority result. ECG also estimates its
moments from finite particles. Comparisons should state the available field
class, fitting budget, and whether correction weights and resampling are
retained; no resampling alone does not mean unweighted Proposal-Only sampling.

\section{Scalar and Directional Calibration}
\label{app:sec:directional}

\subsection{The oracle benefit of directional shrinkage}

Scalar shrinkage is a restricted decision class, not an isotropy assumption.
On the range of $C$, write $\theta=C^{1/2}\lambda^\star$ and
$Z=C^{1/2}\widehat\lambda=\theta+\xi$. Suppose $\E\xi=0$ and define
$H_j=\theta_j^2$, $P_j=\E\xi_j^2$. For fixed diagonal amplitudes,
\begin{equation}
\E\|\operatorname{diag}(\alpha_j)Z-\theta\|^2
=\sum_j\{(1-\alpha_j)^2H_j+\alpha_j^2P_j\}.
\label{app:eq:directional-risk}
\end{equation}
No independence between the coordinates is needed. Cross-covariances do
not enter a sum of coordinatewise squares.

\begin{proposition}[Exact directional oracle gain]
\label{prop:directional-gain}
Write $T_j=H_j+P_j$. If all $T_j=0$, every amplitude has the same risk. Otherwise omit coordinates with $T_j=0$. The directional
oracle is $\alpha_j^*=H_j/T_j$, the scalar oracle is
$\alpha_s^*=\sum_jH_j/\sum_jT_j$, and their risk difference is
\begin{equation}
\mathcal R_{\rm scalar}^*-\mathcal R_{\rm directional}^*
=\sum_jT_j(\alpha_j^*-\alpha_s^*)^2\ge0.
\label{app:eq:directional-gain}
\end{equation}
\end{proposition}
\begin{proof}
Complete each square in Eq.~\eqref{app:eq:directional-risk}. The unconstrained
minimizers lie in $[0,1]$. Under a common amplitude, minimizing the sum gives
$\alpha_s^*$. Subtract the coordinatewise minima to obtain the identity.
\end{proof}

For example, let $\theta=(1,0)$ and $\xi=(0,\varepsilon)$ with
$\E\varepsilon=0$, $\E\varepsilon^2=L>0$. The directional oracle keeps
the first coordinate and removes the second, giving zero coefficient risk.
The scalar oracle has risk $L/(1+L)$. Grouping coordinates replaces
$H_j,P_j$ by their group sums. A zero-amplitude group is a controller
truncation. Thus reduced rank and shrinkage are complementary decisions.
These are oracle comparisons; estimating additional amplitudes has a cost.

\subsection{An exact sample on which directional LOO is worse}

Extra amplitudes can fit deletion noise even when they use the same fits.
Let $a$ be uniform on $\{e_1,-e_1,e_2,-e_2\}$, let $g$ be an independent
uniform sign, and use $\eta=1$, $N=4$. The ordered sample
\[
D=((e_1,-1),(e_1,-1),(e_2,1),(-e_2,1))
\]
has probability $8^{-4}$. Direct calculation gives
$\widehat\lambda_D=(2/5,0)$ and the deleted residual representation
$u_i+v_i^\top\alpha$ with
\begin{equation}
u=(-4/3,-4/3,4/3,4/3)^\top,\qquad
(v_i^\top)_{i=1}^4=
\begin{pmatrix}8/33&0\\8/33&0\\-8/39&16/39\\-8/39&16/39\end{pmatrix}.
\label{app:eq:directional-counterexample}
\end{equation}
The first two deletion coefficients are $(4/11,0)$; the last two are
$(4/13,4/13)$ and $(4/13,-4/13)$, which verifies the matrix using the
deleted means. The scalar criterion has minimizer $143/290$.
For the two-coordinate criterion, at $(1,0)$ both residual types are
nonzero, the derivative in coordinate one is negative, and the derivative
in coordinate two is positive. These are the constrained optimality
conditions at that corner. The two distinct row vectors span the plane,
so strict convexity makes $(1,0)$ the unique minimizer.

Population risk is $V(\lambda)=1+\|\lambda\|^2/2$. Deploying the full fit
with these amplitudes therefore gives
\[
V_{\rm scalar}=1+\frac{2}{25}\left(\frac{143}{290}\right)^2
\simeq1.01945,\qquad V_{\rm directional}=1+\frac{2}{25}=1.08.
\]
This is a positive-probability counterexample to samplewise dominance of
directional LOO. It is not a comparison of their risks averaged over all
possible fitting samples. Together with Proposition~\ref{prop:directional-gain},
it separates the benefit of a richer oracle from the reliability of estimating it.

\subsection{When the bias term can be omitted}
\label{app:sec:bias-conditions}

Projection orthogonality alone does not imply $B=0$. One sufficient
condition is a correctly specified conditional linear model
$g=c-a^\top\lambda^\star+\varepsilon$, with $\E[\varepsilon\mid a]=0$,
full empirical rank, unregularized least squares, and integrable coefficients.
Conditioning on the design gives $\E[\widehat\lambda\mid a_{1:N}]=\lambda^\star$,
so $B=0$. This is stronger than the population normal equation.

A centered first-order CLT is also insufficient by itself. A coefficient
bias of order $N^{-1}$ can make $B$ the same order as the estimation
penalty $P$. The general bound
$|B|\le\sqrt H\|\E\widehat\lambda-\lambda^\star\|_C\le\sqrt{HP}$
follows from Cauchy--Schwarz, but need not be small relative to $P$.
In a noiseless scalar example with ridge contraction
$\widehat\lambda=c\lambda^\star$, $0<c<1$, one has
$P=(1-c)^2H$ and $B=-(1-c)H$. The constrained risk oracle retains the
entire fitted update, whereas $H/(H+P)<1$ shrinks it again.
This proves that ridge can invalidate the zero-bias interpretation even
without sampling noise. The paper's bias audit is empirical evidence for
its audited regime, not a universal justification for dropping $B$.

\section{Specialization to FKC and ACE}

Let expert $i$ provide density $q_t^{(i)}$, score
$s_t^{(i)}=\nabla\log q_t^{(i)}$, and probability-flow velocity $v_t^{(i)}$.
For a valid exponent schedule $\gamma_i(t)$,
\begin{equation}
 p_t(x)=\frac{1}{Z_t}\prod_{i=1}^{I}q_t^{(i)}(x)^{\gamma_i(t)},
 \qquad s_t(x)=\sum_i\gamma_i(t)s_t^{(i)}(x).
 \label{app:eq:composed-path}
\end{equation}
In the implementation's SDE parameterization, the selected probability-flow
component enters
\begin{equation}
 g_t^v=\nabla\!\cdot v_t+v_t\cdot s_t+b_t,
 \label{app:eq:fkc-potential}
\end{equation}
up to a spatial constant, where $\eta_t$ is the scalar linear-drift
coefficient fixed by the underlying diffusion schedule and
\begin{align}
 b_t(x)
 &=\eta_t x\cdot s_t(x)
 +\sum_i\dot\gamma_i(t)\log q_t^{(i)}(x) \notag\\
 &\quad-\sum_i\gamma_i(t)
 \left\{\nabla\!\cdot v_t^{(i)}(x)+
 v_t^{(i)}(x)\cdot s_t^{(i)}(x)\right\}.
 \label{app:eq:b-term}
\end{align}
The omitted divergence $\nabla\!\cdot(\eta_tx)=d\eta_t$ is spatially
constant and irrelevant to normalized weights.

The generic theory requires a valid path. VCG applies directly when the
fixed-exponent path is valid.  If fixed exponents create an invalid or
singular intermediate path, ACE first repairs the schedule while preserving
the desired endpoints \citep{lee2026collapse}. Proposal fitting then acts
on the repaired path. RCPT does not substitute for path repair and does not
change Eq.~\eqref{app:eq:composed-path}.

\section{Implementation Details}
\label{app:sec:implementation}

\subsection{Stable weighted least squares}

For normalized current particle weights $\omega_n$, the code minimizes
\begin{equation}
 \min_{\lambda,c}\sum_{n=1}^{N}\omega_n
 \left(g_n^0+A_n\lambda-c\right)^2+\eta\norm{\lambda}_2^2.
 \label{app:eq:wls}
\end{equation}
The intercept is eliminated by weighted centering.  Covariances and the
$m\times m$ linear solve are computed in float64; singular systems use a
pseudoinverse.  A coefficient clip is available in the high-dimensional
experiments.  After clipping, the code recomputes the selected empirical
variance and uses the anchor coefficient $\lambda=0$ if it increased beyond a
floating-point tolerance.  The chosen proposal is used consistently for state
propagation and all tracked $\log q_i$ updates.

Theorem~\ref{thm:structured-calibration} analyzes the scalar-selected field
before an additional coefficient clip. If clipping changes it to
$\lambda_{\rm final}$, let
$e_{\rm clip}=\|\lambda_{\rm final}-\widehat\alpha\widehat\lambda\|_C$.
The triangle inequality for centered residuals gives
\begin{equation}
\sqrt{V(\lambda_{\rm final})}
\leq\sqrt{V(\widehat\alpha\widehat\lambda)}+e_{\rm clip}.
\label{app:eq:clipping-error}
\end{equation}
An Anchor fallback has exactly the Anchor risk, but an empirical-variance
guard alone is not an out-of-fit guarantee. The realized deployed field and
any clipping therefore need to be distinguished in applying the theory.
The logged amplitude is the scalar selected before componentwise clipping;
if an empirical guard falls back to Anchor, it remains a diagnostic of the
pre-guard decision rather than the final coefficient.

The logged diagnostics are the anchor variance, selected variance, regularized
within-class floor, removable fraction
\begin{equation}
 R_t^2=\frac{\widehat\Var(g_t^0)-\widehat V_{t,*}}
 {\widehat\Var(g_t^0)},
\end{equation}
coefficient norm, conditional ESS before and after selection, conventional ESS,
and resampling events.  Same-particle floors are optimization diagnostics, not
held-out estimates of the population floor.

\subsection{Basis and divergence by domain}

\paragraph{Two dimensions.}
The controlled
ACE experiment anchors at its standard velocity and uses every coordinate
vector field whose scalar coefficient is a monomial of total degree at most
two.  Exact automatic differentiation computes divergence in two dimensions.

\paragraph{Images.}
We use the official pretrained DLSM CIFAR-10 score model and noise-conditioned
classifier at upstream commit \texttt{259c767}. For class $c$, write
$h_c=\log p_\phi(c\mid x,t)$, $u_c=\nabla h_c$, and $b_t=\nabla\log q_t$.
For target $\widetilde p_t=q_t\exp(\beta h_c)$ and matched classifier-guided
VE probability-flow Anchor with $\gamma=\beta$, the raw response in generation
time $s=1-t$ is
\begin{equation}
 g_{\beta,\beta}^{\mathrm{PF}}
 =-\beta\partial_t h_c+\beta\mathsf a(t)b_t^\top u_c
 +\tfrac12\mathsf a(t)\left(\beta\Delta h_c+\beta^2\norm{u_c}_2^2\right),
 \label{app:eq:dlsm-anchor-response}
\end{equation}
where $\mathsf a(t)$ is the forward variance rate. This identity assumes an
exact VE heat-flow base density; learned scores and Hutchinson traces yield
a plug-in approximation. The response is centered for residual fitting.
For the terminal cohort, $\beta=\gamma=2.5$. The implementation represents
this Anchor as a scale-1.5 base plus the same fixed $+1$ target-class
coefficient in every arm. Only the nine remaining fields
$\phi_j=\tfrac12\mathsf a(t)\nabla\log p_\phi(j\mid x,t)$, $j\ne c$,
are fitted, with target Stein features
$\mathcal T_{p_t}\phi_j=\nabla\!\cdot\phi_j+\phi_j^\top(b_t+\beta u_c)$.
Thus zero auxiliary correction recovers the matched scale-2.5 Anchor.
One realized Rademacher Hutchinson probe is shared across methods within each
class--seed--time cell. Fit states are forward-noised training images;
independent test images provide held-out states. Ridge is $10^{-4}$.
Schedules are fitted at 16 midpoint times and linearly interpolated during
500-step deterministic VE probability-flow Euler sampling (PFO-500).
The separate unit-power mechanism audit uses $\beta=\gamma=1$.
Neither terminal images nor evaluation metrics enter fitting or calibration.
Forward-noised labeled data are a proxy for the model-defined target path;
these common-state risks do not directly evaluate its exact expectation.

\paragraph{Molecules.}
States are flattened padded molecular coordinates and attributes, with all
existing projections and masks retained.  In scaffold decoration, each of the
four expert probability-flow velocities is crossed with degree zero, one, and
two bounded relative-density Legendre gates, giving the fixed 12-field basis
in Eq.~\eqref{app:eq:scaffold-basis12}.  Sixteen particles fit the VCG or
RCPT schedule, and that schedule is frozen for independent
deployment.  Finite controller coefficients, feature matrices, state
increments, and residuals are checked before every accepted update; these VCG-based arms
share the stated numerical guards.

\paragraph{Reward-tilted molecules.}
The EEGSDE basis~\citep{bao2023equivariant} begins with
$\nu_t\nabla R_{\rm polarizability,t}$ and
$\nu_t\nabla R_{\rm gap,t}$, where $\nu_t$ is the one-step reverse-transition
variance, and crosses these families with the fixed bounded Legendre gates.
Each resulting field is normalized by its fit-batch RMS over active state
coordinates before regression; the frozen schedule stores these scales. The
base conditional diffusion drift and scaled FKC Anchor are fixed parts of the
proposal and are not counted as controller fields.
Autodifferentiation supplies reward gradients and Stein divergences.  A
same-state time finite difference evaluates both $\dot\beta_{j,t}R_{j,t}$ and
$\beta_{j,t}\partial_tR_{j,t}$ in the reward-path potential.  Generation
uses the training reward networks; all reported properties use independently
trained evaluators, and molecular identities use RDKit canonical isomeric
SMILES.

\subsection{Deletion scope and computational complexity}

Deletion refits in this description act on the evaluated regression features.
When gate statistics or field normalizations are estimated from the same
cohort and then held fixed across folds, this is feature-level deletion rather
than deletion of the entire field-construction pipeline. The fixed-dictionary
statistical interpretation requires that construction to be fixed or
conditioned on with a justified sampling law. The realized-field bound in
Theorem~\ref{thm:structured-calibration} still applies, but its calibration
errors must include any resulting adaptivity. No iid claim is made solely
because the feature matrix has been frozen.

After field values and divergences are available, weighted covariance formation and the VCG solve cost \(O(Nm^2+m^3)\).  \ours{} reuses those features and adds only dense leave-one-out algebra, with no expert or denoiser call.  Hardware-identified timing, peak VRAM, and measured particle scaling are reported in Appendix~\ref{app:sec:cost}.

\begin{algorithm}[htbp]
\caption{Online VCG/\ours{} Sampling Loop}
\label{app:alg:full}
\textbf{Input:} valid path, Anchor, basis, ridge, correction flag, particles $X_0^{[n]}\sim p_0$.
Initialize log weights and tracked expert log densities. At each step:
\begin{enumerate}
\setlength{\itemsep}{0pt}
\item Evaluate target score, Anchor, fields, and divergences; form the rate and feature matrix.
\item Normalize weights and solve Eq.~\eqref{app:eq:wls}. For RCPT, compute deletion residuals and the structured amplitude.
\item Apply the configured coefficient clip and empirical-variance fallback.
\item Propagate under the selected drift and update tracked densities using that same drift.
\item If correction is retained, update FK weights. If the resampling rule triggers, resample states and tracked quantities with common ancestry and reset weights.
\end{enumerate}
Return particles, weights (if retained), and diagnostics.
\end{algorithm}

\subsection{Evaluation laws and dependent particles}
\label{app:sec:reference-transfer}

The theoretical risk is defined under the target marginal $p_t$. Common-state evaluation compares methods on identical states but estimates risk under the reference law; active-path evaluation also changes the state distribution with the proposal. Neither is automatically a target-law expectation. The Checker corrected-SMC reference is an approximation, and the molecular common-state audit does not quantify its gap to the target law.

Equal weights after resampling do not imply independent fitting information. Particles can share ancestors, and deleting one descendant can leave closely related states in the fitting ensemble. Weight ESS therefore does not by itself quantify regression information or justify an online guarantee. Lineage-aware deletion and particle-system concentration require a separate analysis \citep{delmoral2011concentration}.

Data-derived features, coefficient clipping, and empirical fallbacks are part of the deployed estimator. The conditional structured bound evaluates a fixed fitted direction; additional transformations need their own error control, and an empirical Anchor fallback is not an independent out-of-fit certificate. These distinctions delimit the interpretation of the reported risk audits.

\section{Additional Evaluations}
\label{app:sec:experiments}

\subsection{Controlled targets}
\label{app:sec:2d-formulation}

\paragraph{Checker.}
Draw $U\sim\mathrm{Unif}[-2,2]$, $V\sim\mathrm{Unif}[0,1]$, and
$K\sim\mathrm{Bernoulli}(1/2)$, set
\begin{equation}
X_{\rm raw}=2U,\quad
Y=2\{V-2K+(\lfloor U\rfloor\bmod 2)\},\quad
X=(X_{\rm raw}+2\bmod 8)-4,
\label{app:eq:checker-definition}
\end{equation}
and condition on $A=\mathbf1\{X\geq0\}=1$ and
$B=\mathbf1\{X+Y\geq0\}=1$.  This gives the exact alternating checker law
on $[-4,4]^2$.  Three separately trained expert paths define the ACE target.
Anchor is only the learned joint expert velocity
$v_{XY\mid B}$; it is not the sum of all three expert velocities.  The other
paths enter the time-dependent multiplicative target, while VCG and RCPT add a
12-field coordinate-polynomial correction (six monomials of total degree at
most two for each of two velocity coordinates) to the shared Anchor.  At every
fitted time, Anchor, VCG, and RCPT share the same target, expert
evaluations, integration grid, ridge convention, and random streams.  A
schedule fitted on 16 particles is frozen before 2,000 fresh Proposal-Only or SMC trajectories are generated in the primary comparison (Table~\ref{tab:six-baseline-2d}). All runs use 200 integration steps, ridge $10^{-4}$, SMC ESS threshold $0.7$, and the expert schedule \texttt{default\_linear / 1-t**2 / cos\_t}. The calibration ablation uses these same seeds and deployments. Its residual diagnostic is an active-path integral, not an independent target-law risk estimate.

\paragraph{Illustrative replay and calibration ablation.}
The main Checker figure uses one seed, 16 fitting particles, 200 integration steps,
and the exact common-endpoint schedule \texttt{default\_linear / 1-t**2 / cos\_t}.
Each frozen schedule generates 10,000 particles. Endpoint estimates use fixed
2,000-particle audit indices; Proposal-Only has no FK weights, resampling, or
terminal reweighting. The accompanying SMC replay retains correction.
\begin{table}[!htb]
\caption{\textbf{Single-seed Checker figure audit.} This replay is illustrative, not an additional ten-seed estimate. Residual is an active-path integral and is not pooled across deployment modes.}
\label{app:tab:checker-new}
\centering\small
\setlength{\tabcolsep}{6pt}
\begin{tabular}{@{}lrrrr@{}}
\toprule
Mode & Controller & $W_2$ & MMD & Residual\\
\midrule
Proposal & Anchor & 0.7689 & 0.0645 & 668.21\\
Proposal & VCG & 1.1933 & 0.2572 & 3087.30\\

Proposal & ECG & 1.3156 & 0.4725 & 291.44\\
Proposal & \ours{} & 0.6819 & 0.0670 & 30.03\\
\midrule
SMC & Anchor & 1.3888 & 0.1387 & 48.64\\
SMC & VCG & 5.7207 & 6.9463 & 5673.08\\

SMC & ECG & 1.7622 & 0.4858 & 152.01\\
SMC & \ours{} & 0.9445 & 0.1235 & 40.87\\
\bottomrule
\end{tabular}
\end{table}\begin{table}[!htb]
\caption{\textbf{Checker calibration ablation.} The same ten paired seeds, 16 fitting particles, fixed schedule, and 2,000-particle deployments as Table~\ref{tab:six-baseline-2d}. Values are means. Direct LOO replaces the structured amplitude criterion. Residual is the integrated active-path diagnostic, not independent target-law risk; its Direct LOO Proposal-Only mean is driven by one severe run. Ancestors counts terminal SMC lineages out of 2,000; without resampling every trajectory retains its own ancestor.}
\label{app:tab:checker-direct}
\centering\small
\setlength{\tabcolsep}{4pt}
\begin{tabular}{@{}llrrrrr@{}}
\toprule
Mode & Controller & $W_1\downarrow$ & $W_2\downarrow$ & MMD$\downarrow$ & Residual & Ancestors\\
\midrule
Proposal & Anchor & 0.474 & 0.646 & 0.048 & 614.37 & 2000.0\\
 & Direct LOO & 0.465 & 0.636 & 0.068 & 1777.69 & 2000.0\\
\rowcolor{oursrow} & \ours{} & 0.442 & 0.575 & 0.073 & 103.24 & 2000.0\\
\midrule
SMC & Anchor & 0.367 & 0.510 & 0.048 & 31.56 & 338.7\\
 & Direct LOO & 0.435 & 0.631 & 0.059 & 25.62 & 630.7\\
\rowcolor{oursrow} & \ours{} & 0.316 & 0.464 & 0.042 & 25.09 & 568.5\\
\bottomrule
\end{tabular}
\end{table}

In this calibration ablation, RCPT has lower mean $W_1$ and $W_2$ than Direct LOO in both deployment modes and lower mean MMD under SMC. Direct LOO has lower Proposal-Only MMD and retains more SMC ancestors on average ($630.7$ versus $568.5$). These distinct rankings separate endpoint accuracy from genealogy.

\subsection{Common-fit realized-ray and adaptive-regularization audit}
\label{app:sec:ray-audit}

\paragraph{Protocol and evaluation law.}
This audit isolates the calibration decision on identical fitting states.
Thirty raw-VCG pilot fits use the 12-field Checker dictionary, fitting
populations $N\in\{8,16,32\}$, and ten paired seeds per budget. At each of 200 native steps,
RCPT and Direct LOO use the same fitted direction. LOO-selected ridge and
PCR use the same fitting states; PCR recomputes its basis within deletion
folds. Selection uses fitting residuals only. No endpoint metric selects a
ridge level, rank, or amplitude. Coefficient clipping and calibration
fallback are disabled in this audit, so scalar rules remain on the fitted
ray. This is a separate cohort from the primary fixed-schedule comparison
and its matched calibration ablation.

Independent reference states separate calibration from evaluation. Intermediate
normalized target densities are unavailable for the learned expert paths.
We therefore use two independent corrected Anchor-SMC systems of 16,384
particles each, with separately normalized weights. Both systems resample
five times; minimum pre-resampling ESS is 9,102--9,229. Their combined
32,768-particle weighted law $\mu_t$ defines the reference risk. This is
neither exact target integration nor iid target sampling. Disagreement
between systems measures simulation sensitivity; it does not bound shared
model bias or the target-reference covariance gap in
Appendix~\ref{app:sec:reference-transfer}.

The audit uses 21 fixed rounded indices from
$\operatorname{linspace}(0,199,21)$. Each native step is assigned to its
nearest audit node; the resulting cell counts divided by 200 give time
weights $w_k$. Reported Risk/Anchor is the ratio of
$\sum_k w_k\Var_{\mu_{t_k}}(g^0_{t_k}+a_{t_k}^\top\lambda_{t_k})$
to the corresponding Anchor risk, summarized over the ten fitting seeds.
Bootstrap intervals resample paired fitting seeds, not time nodes. They
describe variation over fits conditional on this reference construction.

Table~\ref{app:tab:ray-risk} in the main text reports this comparison.

\paragraph{Adaptive regularization.}
RCPT has the lowest mean reference risk among the implemented rules at
all three fitting populations (Table~\ref{app:tab:ray-risk}). At $N=16$,
its paired Risk/Anchor differences from Direct LOO, ridge, and PCR are
$-0.6331$, $-0.3375$, and $-0.6804$, with respective 95\% bootstrap intervals
$[-1.363,-0.0669]$, $[-0.7484,-0.0510]$, and $[-1.182,-0.2518]$.
At $N=8$ and 32, the intervals against Direct LOO and ridge include zero;
the intervals against PCR remain below zero. These are pairwise intervals
over ten fits, not simultaneous guarantees. The results support scalar
calibration at low particle budgets without establishing dominance over
all directional estimators. PCR changes a learned basis and rank; it does
not exhaust grouped or coordinatewise shrinkage.

\paragraph{Realized calibration.}
For each fit and node, replace $p$ by $\mu_t$ in
Eq.~\eqref{app:eq:realized-ray-risk} to obtain $q_\mu$, $\kappa_\mu$, and
the reference-ray oracle. Direct variance evaluation agrees with the
quadratic at $\alpha\in\{0,0.25,0.5,0.75,1\}$ to maximum absolute error
$2.328\times10^{-10}$. None of the 630 fit--time cells requires the
near-zero-curvature convention, whose tolerance is
$\max\{10^{-14},10^{-10}\max(V_\mu(0),1)\}$.

The structured ingredients are not quantitatively calibrated pointwise.
Median $S/q_\mu$ rises from $0.105$ to $0.276$ and $0.606$ as $N$ grows;
$\kappa_\mu<0$ in 56/630 cells (Table~\ref{app:tab:ray-calibration}).
A nonnegative penalty cannot match those negative covariances exactly.
Nevertheless, mean integrated ray regret is $8.359$, $4.500$, and $3.776$
for RCPT, versus $12.900$, $27.960$, and $4.635$ for Direct LOO.
Direct LOO has smaller mean absolute amplitude error at $N=8$ and 32,
but not smaller regret. Curvature weights amplitude errors differently
across times, and boundary optima add a one-sided term. This distinction
explains why unweighted amplitude accuracy and decision quality need not
have the same ordering.

\begin{table}[!htb]
\caption{\textbf{Realized quadratic diagnostics.}
All quantities use the independent corrected-SMC reference law.
Spearman coefficients summarize the seed-level $S$--$q_\mu$ and
$\widehat P$--$\kappa_\mu$ diagnostics. Their limited agreement precludes
interpreting the structured ingredients as accurate pointwise estimates.}
\label{app:tab:ray-calibration}
\centering\small
\setlength{\tabcolsep}{7pt}
\begin{tabular}{@{}rrrrr@{}}
\toprule
$N$ & Median $S/q_\mu$ & $\kappa_\mu<0$ & $\rho(S,q_\mu)$ & $\rho(\widehat P,\kappa_\mu)$\\
\midrule
8 & 0.105 & 4.8\% & 0.418 & 0.406\\
16 & 0.276 & 6.2\% & -0.188 & 0.103\\
32 & 0.606 & 15.7\% & 0.115 & 0.673\\
\bottomrule
\end{tabular}
\end{table}

\paragraph{What the oracle calculation establishes.}
The expected-risk oracle and the realized calibration bound answer
different questions. Applying the $H,P,B$ identity to the empirical
distribution of ten fits gives an oracle gain over the better endpoint
larger than the two-system sensitivity in 22/63 population-size--time
cells: 6/21, 6/21, and 10/21. This diagnoses
Theorem~\ref{thm:shrinkage}'s expected-risk calculation on the reference
law. It is not the frequency with which
Theorem~\ref{thm:structured-calibration}'s plug-in error bound is
informative. The latter additionally requires control of realized
approximation errors and target-reference transfer. No finite-sample
certificate for the implemented selector follows from this audit.

\paragraph{Endpoint deployment and completeness.}
Each frozen schedule generates 10,000 paired trajectories in Proposal-Only
and corrected-SMC modes. Exact discrete optimal transport estimates $W_2$
on a predeclared paired 2,000-sample metric subset.
Table~\ref{app:tab:ray-endpoints} retains all attempted denominators.
RCPT completes 60/60 deployments, compared with 45/60 for VCG and 59/60
for PCR. Means for incomplete arms condition on completion and are not
unconditional performance comparisons.

\begin{table}[!htb]
\caption{\textbf{Common-fit Checker endpoints.} Mean $W_2\pm$ standard
error across completed seeds. Finite gives the completion count over every planned seed;
lineages are SMC terminal ancestors out of 10,000, averaged over completed
runs. The reference-risk analysis includes all 30 fits per method, including those whose endpoint deployment failed.}
\label{app:tab:ray-endpoints}
\centering\footnotesize
\setlength{\tabcolsep}{3.2pt}
\begin{tabular}{@{}clrrrrr@{}}
\toprule
& & \multicolumn{2}{c}{Proposal-Only} & \multicolumn{3}{c}{Corrected SMC}\\
\cmidrule(lr){3-4}\cmidrule(lr){5-7}
$N$ & Controller & $W_2\downarrow$ & Finite & $W_2\downarrow$ & Finite & Lineages\\
\midrule
8 & Anchor & $0.7725\pm0.0136$ & 10/10 & $0.5773\pm0.0290$ & 10/10 & 1723\\
& VCG & $1.594\pm0.1779$ & 5/10 & $2.316\pm0.3144$ & 8/10 & 613.9\\
& Direct LOO & $0.9814\pm0.0422$ & 10/10 & $0.7267\pm0.0735$ & 10/10 & 1642\\
& LOO ridge & $1.082\pm0.0467$ & 10/10 & $0.7458\pm0.0871$ & 10/10 & 1631\\
& LOO PCR & $1.073\pm0.0755$ & 9/10 & $0.8080\pm0.1280$ & 10/10 & 1559\\
\rowcolor{oursrow}& \ours{} & $0.8360\pm0.0447$ & 10/10 & $0.6417\pm0.0372$ & 10/10 & 1873\\
\midrule
16 & Anchor & $0.7973\pm0.0159$ & 10/10 & $0.6602\pm0.0337$ & 10/10 & 1893\\
& VCG & $1.423\pm0.5909$ & 3/10 & $3.744\pm0.2502$ & 9/10 & 657.2\\
& Direct LOO & $0.9174\pm0.1168$ & 10/10 & $0.7979\pm0.1391$ & 10/10 & 1542\\
& LOO ridge & $0.9675\pm0.1051$ & 10/10 & $0.6963\pm0.0639$ & 10/10 & 1832\\
& LOO PCR & $0.9682\pm0.1157$ & 10/10 & $0.7585\pm0.0634$ & 10/10 & 1804\\
\rowcolor{oursrow}& \ours{} & $0.9062\pm0.1097$ & 10/10 & $0.7341\pm0.1011$ & 10/10 & 1592\\
\midrule
32 & Anchor & $0.8044\pm0.0188$ & 10/10 & $0.5990\pm0.0310$ & 10/10 & 1895\\
& VCG & $0.7991\pm0.0545$ & 10/10 & $1.731\pm0.5101$ & 10/10 & 2400\\
& Direct LOO & $0.7918\pm0.0486$ & 10/10 & $0.6553\pm0.0218$ & 10/10 & 2861\\
& LOO ridge & $0.8054\pm0.0455$ & 10/10 & $0.6068\pm0.0463$ & 10/10 & 3203\\
& LOO PCR & $0.8023\pm0.0390$ & 10/10 & $0.6067\pm0.0463$ & 10/10 & 3138\\
\rowcolor{oursrow}& \ours{} & $0.7851\pm0.0482$ & 10/10 & $0.6806\pm0.0299$ & 10/10 & 2671\\
\bottomrule
\end{tabular}
\end{table}

Residual improvement does not uniformly improve terminal transport.
RCPT has lower mean Proposal-Only $W_2$ than Direct LOO, ridge, and PCR
at each budget, but SMC favors ridge at $N=16$ and both ridge and PCR at
$N=32$. Paired RCPT-minus-Anchor $W_2$ intervals include zero in
Proposal-Only at every budget and in SMC at $N=8,16$.
At $N=32$, the SMC difference is $+0.0817$, with 95\% bootstrap interval
$[+0.0051,+0.1629]$. Reference risk, propagated transport energy, and
corrected finite-particle endpoints are different quantities; the
one-sided transport bounds do not order these observed metrics.

\paragraph{Selection cost.}
Mean selector time per 200-step schedule is $0.132$ seconds for RCPT,
$1.31$ for Direct LOO, $9.34$ for ridge selection, and $13.03$ for PCR.
These are selector timings within this implementation, not total sampling
times or hardware-independent complexity comparisons. Together with the
risk results, they support scalar shrinkage as an inexpensive decision
under limited fitting information. They do not establish its global
optimality among regularizers.

\subsection{Scaffold protocol}
\label{app:sec:scaffold-revised}

The fixed benchmark contains all 76 CrossDocked2020 scaffold--pocket
tasks~\citep{francoeur2020crossdocked}.  Its four pretrained expert paths are
EDM for the fragment and full ligand, GeoDiff for fragment conformations, and
pocket-conditioned DiffSBDD for the full ligand
\citep{hoogeboom2022equivariant,xu2022geodiff,schneuing2024diffsbdd}.
For each expert velocity $v_i$, the controller contains
\begin{equation}
 \begin{aligned}
 \phi_{i,0}=v_i,\qquad
 \phi_{i,1}=z_i v_i,\qquad
 \phi_{i,2}=\tfrac12(3z_i^2-1)v_i,\\
 r_i=\log q_i-I^{-1}\sum_{\ell=1}^I\log q_\ell,\\
 z_i=\tanh\!\left(\frac{r_i-\mu_i}{\sigma_i}\right),
 \end{aligned}
 \label{app:eq:scaffold-basis12}
\end{equation}
giving 12 relative-density-gated fields with bounded gates.  The degree-zero fields are
exactly the original mechanistic expert velocities.  The gate statistics are
fixed from the fitting cohort and never selected using decoded chemistry.

Each pocket uses 16 fitting particles with one fitting seed.  Anchor is the natural
ACE proposal; VCG is the ordinary 12-field regression; RCPT applies exact
weighted leave-one-out refitting at the same fitted times.  The resulting schedules are
frozen and deployed on 16 fresh paired trajectories per arm.  Proposal-Only
deployment retains neither correction weights nor resampling.  RDKit decoding,
scaffold retention, uniqueness, diversity, QED, SA, Lipinski summaries, and
QuickVina docking are computed only after deployment and never enter controller
selection.  Affinity-improved yield uses the fixed generation denominator;
``unique useful'' additionally deduplicates valid, scaffold-retaining,
reference-beating molecules. Docking success alone counts any candidate with
a successful docking result, even if it fails the separate chemical-validity
check, divided by all generated slots.  Paired pocket bootstrap intervals treat the 76
pockets as statistical units.  The main table uses the original complete
fit-16 deployment; an independent 64-trajectory replication reuses the same
frozen schedules and increases evaluation to 4,864 candidates per arm.

The common-state audit evaluates all 76 fitted schedules on identical
shared reference states, rather than their active deployment paths. VCG exceeds
Anchor risk in 75/76 pockets. RCPT improves on VCG in all 76 and repairs 72
of those 75 harmful updates; Direct LOO repairs 73. Mean Risk/Anchor is
25.9641 for VCG, 0.8935 for RCPT, and 0.7958 for Direct LOO. Direct LOO beats
RCPT in 73/76 pockets; the paired RCPT-minus-Direct difference is 0.0977
(95\% bootstrap CI [0.0622,0.1197]). The worst ratios are 1491.65, 2.00, and
3.02, respectively: the mean improvement and tail comparison differ.

\begin{table*}[!htb]
\caption{\textbf{High-precision 76-pocket scaffold deployment.}  The fixed
fit-16 schedules are independently deployed with 64 candidate slots per arm
and pocket (4,864 per arm).  This table isolates the VCG-to-RCPT comparison; downstream means remain close to Anchor. Risk is normalized by Anchor. Dock counts successful docking over all generated candidates, irrespective of chemical validity. Bold marks the better of the two reported controllers.}
\label{app:tab:scaffold-deploy64-confirmation}
\centering
\small
\setlength{\tabcolsep}{3.0pt}
\begin{tabular}{@{}lrrrrrr@{}}
\toprule
Controller & Validity$\uparrow$ & Scaffold$\uparrow$ & Risk/Anchor$\downarrow$ &
Dock success$\uparrow$ & Affinity$\uparrow$ & Unique useful$\uparrow$\tabularnewline
\midrule
VCG & 0.216 & 0.212 & 23.822 & 0.569 & 0.148 & 0.101\tabularnewline
\rowcolor{oursrow}\ours{} & \textbf{0.433} & \textbf{0.415} &
\textbf{0.878} & \textbf{0.904} & \textbf{0.292} & \textbf{0.224}\tabularnewline
\bottomrule
\end{tabular}
\end{table*}

Table~\ref{app:tab:scaffold-direct} reports this comparison.

Direct LOO lowers active-path Risk/Anchor by 0.1199 relative to RCPT
(95\% CI [$-0.1293,-0.1099$]), but its validity difference is $-0.0206$
([$-0.0485,0.0074$]) and unique useful yield difference is $-0.0148$
([$-0.0419,0.0115$]). All reported downstream paired intervals include zero.
The molecular implementation includes coefficient clipping and an empirical
fallback, so these deployed arms include the guards in
Appendix~\ref{app:sec:implementation}. The bundled three-later-seed Direct summary
reports validity 0.4252, scaffold yield 0.4068, Risk/Anchor 0.7576, docking 0.9191,
affinity yield 0.2944, and unique useful yield 0.2536. It is a different
aggregation from Table~\ref{app:tab:scaffold-direct}; matched three-seed
control summaries are unavailable, so it is not used for paired comparisons.

The separate ECG run covers the same 76-pocket manifest with 16 fitting
particles and 16 fresh outputs per pocket. The implementation requires the
\texttt{scores} basis: $\phi_i=\log q_t^{(i)}$ and
$\nabla\phi_i=\nabla\log q_t^{(i)}$, one potential per expert. It solves the
weighted Ritz system with centered potential--response covariance and the
uncentered gradient Gram matrix, then applies the configured coefficient clip.
It does not reject a finite ECG coefficient merely because its empirical VCG
residual increases. This four-potential class differs from the 12 gated
velocity fields; the endpoint comparison therefore assesses the instantiated
proposal methods, not two optimizers on an identical field class.
Table~\ref{app:tab:scaffold-budget} reports the separate controller-capacity ablation.
Against VCG-4, RCPT improves validity by $0.0510$ at fit 16 (paired 95\%
bootstrap interval $[0.0041,0.0987]$) and $0.0872$ at fit 64
($[0.0395,0.1349]$). At fit 64, its affinity-qualified and unique useful
yields improve by $0.0699$ ($[0.0329,0.1086]$) and $0.0650$
($[0.0288,0.1020]$), respectively. These gains establish downstream benefits
beyond reducing residual risk; VCG-4 actually has lower mean residual risk.
At fit 16, the unique useful yield gain is $0.0255$ with interval
$[-0.0148,0.0666]$, so that particular gain is not statistically resolved.

\subsection{QM9 steering protocol}
\label{app:sec:reward}

The pretrained EEGSDE model is conditioned on dipole moment and steered toward
polarizability and HOMO--LUMO-gap targets.  The 12-field class crosses the two
reward-gradient families with bounded time/state Legendre gates; all arms use
the same model, target, evaluator, resampling rule, and total trajectory budget.
Reported properties come from independently trained evaluators, not the
guidance networks.  Molecular identities are RDKit canonical isomeric SMILES.

The \emph{split} protocol uses 16 particles to fit one schedule and eight
independent SMC deployment batches of 16, for 144 trajectories in total.  The
\emph{online} protocol uses one live population of 144: coefficients are
refitted from that current population, and the same particles continue and
become outputs. Fitting uses all 144 current particles before the next stochastic
transition. The serialized \texttt{fit\_particles=16} option is inactive in
online mode. Hence online is a strong adaptive batch-generation baseline,
whereas split measures out-of-fit reuse of a frozen controller.  Minimum ESS
and surviving lineage counts are divided by batch size in the combined table,
but one 144-particle genealogy is not statistically identical to eight
independent 16-particle genealogies.  Both protocols contain eight paired fit
seeds and use leave-one-out refitting (mean trust 0.203 split and 0.945 online).

\begin{table}[!htb]
\caption{QM9 steering split means $\pm$ sample SD over eight paired fits.
Unique success counts distinct valid identities meeting both targets at the
unchanged one-MAD threshold per 128 generation slots.  Minimum ESS and
lineages are reported per 16-particle deployment batch.}
\label{app:tab:reward-split-sd}
\centering
\small
\setlength{\tabcolsep}{4pt}
\begin{tabular}{@{}lrrrr@{}}
\toprule
Split method & Eval. MAE & Unique success /128 & Min ESS & Lineages\\
\midrule
Anchor & $0.9171\pm0.0413$ & $20.75\pm6.27$ & $4.023\pm0.656$ & $5.297\pm1.44$\\
VCG & $0.9281\pm0.114$ & $13.88\pm5.79$ & $2.624\pm1.14$ & $2.578\pm1.49$\\
\rowcolor{oursrow}\ours{} & \textbf{$0.8997\pm0.0746$} &
\textbf{$23.00\pm8.05$} & \textbf{$4.207\pm1.54$} &
\textbf{$5.641\pm2.65$}\\
\bottomrule
\end{tabular}
\end{table}

\paragraph{ECG proposal comparison.}
The frozen ECG baseline uses the scalar potentials
\begin{equation}
\phi_{r,\ell}(x,t)=\nu_t\int_0^{e_r(x,t)}-sP_\ell(\tanh s)\,\dd s,
\qquad r\in\{\mathrm{polarizability},\mathrm{gap}\},\quad \ell=0,\ldots,5,
\label{app:eq:reward-ecg-potentials}
\end{equation}
where $e_r$ is target minus prediction in normalized property units and
$\nu_t$ is the one-step reverse-transition variance. Thirty-two-node
Gauss--Legendre quadrature evaluates the potentials; analytic gradients
supply the 12 gated fields. Potentials and fields share fit-batch RMS
normalization. At each fit, $(\widehat G+10^{-8}I)\theta
=\widehat{\Cov}_w(g^0,\phi)$, with
$\widehat G_{ij}=\E_w[\nabla\phi_i^\top\nabla\phi_j]$.
Coefficients are clipped to $\pm2\times10^{-4}$ without a residual-risk
acceptance rule. Sixteen particles fit a frozen schedule before eight
independent 16-particle deployment batches, using eight seeds and 1,000 steps.

\begin{table}[!htb]
\caption{\textbf{ECG completion and diversity diagnostics.} Frozen/deploy
values use eight seeds; $\pm$ denotes sample SD. The reduced online
DriftLite-style configuration has one completed seed. ESS and lineages are
absolute counts per batch. Success is distinct from identity-level unique success.}
\label{app:tab:qm9-ecg-diagnostics}
\centering\footnotesize
\setlength{\tabcolsep}{3pt}
\begin{tabular}{@{}lrrrrrrr@{}}
\toprule
ECG configuration & Seeds & MAE$\downarrow$ & Success$\uparrow$ & Valid & Unique & Min ESS & Lineages\\
\midrule
Frozen/deploy, 12 fields & 8 & $0.8500\pm0.1858$ & $0.4049\pm0.2276$ & 0.9150 & 0.0761 & 1.0 & 1.0\\
Reduced online & 1 & 1.2414 & 0.0000 & 1.0000 & 0.0347 & 1.0 & 1\\
\bottomrule
\end{tabular}
\end{table}

ECG's property scores conceal severe collapse in frozen deployment.
Every batch retains exactly one ancestor. Mean uniqueness is 0.0761,
and mean resampling count is 1,105.5 per seed across eight batches,
or 138.19 per batch. The reported deployed residual ratio is approximately
$3.58\times10^5$. The 0.4049 evaluator success statistic counts repeated
outputs; it cannot be substituted for the number of distinct successful
molecules. Table~\ref{tab:reward-qm9} therefore reports distinct successful identities and
lineage fractions, while Table~\ref{app:tab:reward-split-sd} retains the
available identity-level counts for Anchor, VCG, and RCPT.

The original 12-field online ECG did not complete because its controller
geometry became nonfinite. A reduced implementation closest to the
DriftLite online construction completed one seed, shown in
Table~\ref{app:tab:qm9-ecg-online}. Its error exceeds Anchor and it has zero
success. This is a configuration-specific feasibility result, not an
eight-seed estimate or a general ordering of ECG and VCG objectives.

\begin{table}[!htb]
\caption{\textbf{Completed online ECG audit.} One paired seed, $N=144$.
The Anchor, VCG, and RCPT values come from the same evaluation run;
the reduced ECG result is from the accompanying online audit.
The failed 12-field ECG run is not assigned a finite score.}
\label{app:tab:qm9-ecg-online}
\centering\small
\setlength{\tabcolsep}{7pt}
\begin{tabular}{@{}lrrr@{}}
\toprule
Controller & Eval. MAE$\downarrow$ & Success$\uparrow$ & Uniqueness$\uparrow$\\
\midrule
Anchor & 1.1307 & 0.0000 & 0.0069\\
VCG & 0.8422 & 0.3736 & \textbf{1.0000}\\
ECG & 1.2414 & 0.0000 & 0.0347\\
\rowcolor{oursrow}\ours{} & \textbf{0.8379} & \textbf{0.3788} & \textbf{1.0000}\\
\bottomrule
\end{tabular}
\end{table}

\paragraph{Online calibration ablation.}
Direct LOO replaces the structured amplitude in the same 12-field online
pipeline. Its eight-seed cohort has no matched frozen/deploy run.
Table~\ref{app:tab:qm9-calibration} isolates this change from the main
proposal comparisons. One Direct LOO seed resamples once and retains 84
of 144 ancestors; the other seven retain all 144. Direct-minus-RCPT MAE is
0.0025 (95\% paired bootstrap CI [$-0.0017,0.0074$]). The endpoint difference
is unresolved at this precision. Runtime is not compared across separately
executed cohorts.

Table~\ref{app:tab:qm9-calibration} reports the online calibration ablation.

\paragraph{Regularized Direct ablation.}
Write the exact deleted-loss quadratic as $L_D(\alpha)=A+2b\alpha+c\alpha^2$
and the structured amplitude as $\alpha_S$. The exploratory selector minimizes
\begin{equation}
\alpha_R=\arg\min_{\alpha\in[0,1]}
\{L_D(\alpha)+\gamma c(\alpha-\alpha_S)^2\}.
\label{app:eq:regularized-direct}
\end{equation}
For $c>0$, its unconstrained solution is
$(-b/c+\gamma\alpha_S)/(1+\gamma)$; clipping to $[0,1]$ is performed once.
The run uses $\gamma=1$ and disables componentwise coefficient clipping.
Numerically vanishing curvature selects Anchor.

\begin{table}[!htb]
\caption{\textbf{Exploratory regularized Direct LOO.} One fresh online seed, $N=144$. This tests an additional regularized selector, not the main RCPT method; its small error advantage is not replicated evidence.}
\label{app:tab:reward-regularized}
\centering\small
\setlength{\tabcolsep}{4pt}
\begin{tabular}{@{}lrrrrr@{}}
\toprule
Controller & MAE & Success & Unique /144 & Resamples & Ancestors\\
\midrule
Anchor & 1.3244 & 0.0000 & 0 & 70 & 1\\
VCG & 0.8132 & 0.4541 & 45 & 0 & 144\\

Direct LOO & 0.8080 & 0.4441 & 46 & 0 & 144\\
Reg. Direct ($\gamma=1$) & 0.8044 & 0.4538 & 46 & 0 & 144\\
\rowcolor{oursrow}\ours{} & 0.8099 & 0.4594 & 46 & 0 & 144\\
\bottomrule
\end{tabular}
\end{table}

\subsection{CIFAR-10 class-conditional image generation}
\label{app:sec:image-dlsm}

\paragraph{Mechanism and physical deployment.}
The pretrained DLSM models are fixed. The unit-power mechanism study uses
$\beta=\gamma=1$, all ten classes, and five fitting seeds. Nested fitting
cohorts contain 8, 16, or 32 forward-noised training images at 16 midpoint
times; independent test cohorts contain 128 states per class and time.
Models, states, and realized Hutchinson probes are paired across methods.
The corrected response is Eq.~\eqref{app:eq:dlsm-anchor-response}.
This audit is separate from the $\beta=\gamma=2.5$ terminal cohort.

\begin{table}[!htb]
\caption{\textbf{CIFAR-10 unit-power mechanism study.} Fifty class--seed cells per fit budget, $\beta=\gamma=1$. Risk is held-out Risk/Anchor; harm means a ratio above one. RCPT improves on VCG in 46/50, 47/50, and 43/50 cells at fits 8, 16, and 32.}
\label{app:tab:image-dlsm}
\centering\small
\setlength{\tabcolsep}{8pt}
\begin{tabular}{@{}llrr@{}}
\toprule
Fit & Controller & Mean Risk/Anchor$\downarrow$ & Harm$\downarrow$\\
\midrule
8 & VCG & 1.140 & 46/50\\
\rowcolor{oursrow}8 & \ours{} & 0.997 & 13/50\\
\midrule
16 & VCG & 2.164 & 46/50\\
\rowcolor{oursrow}16 & \ours{} & 1.090 & 22/50\\
\midrule
32 & VCG & 2.695 & 36/50\\
\rowcolor{oursrow}32 & \ours{} & 1.522 & 23/50\\
\bottomrule
\end{tabular}
\end{table}

RCPT reduces mean risk and the number of harmful cells at all three budgets.
Protection remains incomplete: its mean risk exceeds Anchor at fits 16 and 32.
A separate five-seed, fit-16 audit matches the terminal target and Anchor
($\beta=\gamma=2.5$; Table~\ref{app:tab:image-matched-risk}). RCPT has only
5/50 harmful cells, compared with 22/50 for VCG-9 and 13/50 for Direct LOO,
and the smallest worst ratio among fitted methods. VCG-4 has the lowest mean
risk; RCPT does not improve on VCG-9's mean in this setting. All deviations
from Anchor are small. These common-state surrogate risks describe behavior
under the terminal configuration; they are not pooled with the unit-power audit.

\paragraph{Metrics and cohort boundaries.}
Each generation cohort contains 32 images per class and seed. A separate
ResNet-20 supplies independent accuracy and classwise polynomial-kernel KID
against 256 test images per class. DLSM accuracy uses the guidance classifier itself, so it measures agreement
with the guider rather than independent class fidelity. We therefore use
ResNet accuracy as the independent classification test and interpret DLSM
accuracy as a guider self-consistency diagnostic. Clean-FID uses Inception-v3 and all 10,000 test images. DINOv2-small
uses normalized CLS features: FD is pooled, and classwise RBF MMD uses a
target-only median-distance bandwidth. CLIP ViT-B/32 class margin is target
similarity minus the largest non-target similarity, averaged over four fixed
prompt templates. None enters controller fitting. Pooled FD/FID and means of
per-seed FD/FID differ substantially at these sample sizes and are kept apart.

\paragraph{Image-generation protocol.}
The terminal comparison uses seed 261012, target power $\beta=2.5$, and
matched Anchor guidance $\gamma=2.5$. Each method generates 32 outputs per
class, or 320 images. The protocol uses fit 16/evaluate 128 states at 16
schedule nodes and PFO-500, with no Langevin corrector. All methods retain
the same fixed $+1$ target-class coefficient on the scale-1.5 base; VCG-9,
Direct LOO, and RCPT fit the remaining nine auxiliary fields. Thus the
Anchor--RCPT comparison isolates the fitted auxiliary correction under a
shared target-class boost. VCG-4 restricts the fitted correction to four
fitting-covariance directions. Table~\ref{app:tab:image-scale15} reports the
calibration ablation; Table~\ref{app:tab:image-matched-risk} provides the
replicated capacity diagnostic. The single-seed endpoint comparison does
not estimate between-fit uncertainty.

\paragraph{Endpoint interpretation.}
RCPT improves six of seven metrics over matched Anchor and ties DLSM
accuracy. Independent accuracy increases from 0.6625 to 0.6750 (four
additional correct predictions), KID falls from 0.2688 to 0.2407, and
Clean-FID falls from 89.043 to 88.232. RCPT also improves six metrics over
VCG-9, which has higher DLSM accuracy. Relative to Direct LOO, RCPT has
higher independent accuracy and CLIP margin and lower Clean-FID; Direct
LOO has better DLSM accuracy, KID, DINO FD, and classwise DINO MMD.
Thus RCPT improves a broad range of endpoints without uniformly dominating
the direct calibration ablation. These are observed single-seed comparisons,
not seed-level significance claims. FID and DINO FD compare Gaussian feature
summaries; they are not the image-space $W_2$ in Theorem~\ref{thm:terminal}.

\paragraph{Qualitative contrasts.}
Figure~\ref{app:fig:cifar-qualitative} shows three class--sample
pairs from the terminal cohort. RCPT's evaluator prediction
matches the target in each example, while Anchor and VCG-9 do not.
The panel illustrates the corresponding classification changes.

\begin{figure}[!htb]
\centering
\includegraphics[width=0.8\textwidth]{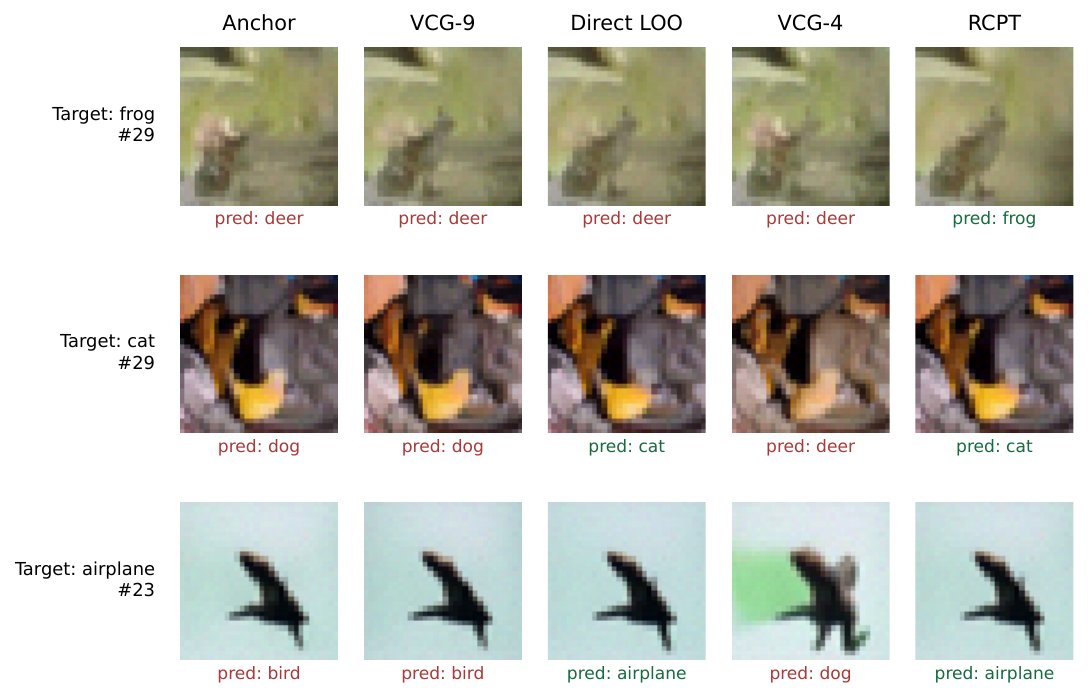}
\caption{\textbf{CIFAR-10 classification contrasts.} Three class--sample pairs
from seed 261012 under the matched $\beta=\gamma=2.5$ protocol. Labels are
independent evaluator predictions. Direct LOO is the calibration ablation;
VCG-4 is the capacity ablation. Aggregate endpoint metrics use all 320 outputs.}
\label{app:fig:cifar-qualitative}
\end{figure}

\subsection{Ten-pocket scaffold fixed-rule comparison}
\label{app:sec:scaffold-simple-baselines}

We compare eleven prespecified controllers on the same ten scaffold--pocket
targets.  Every arm uses one $N_{\rm fit}=16$ fit (one fitting seed) and a paired,
independent $N_{\rm deploy}=16$ Proposal-Only deployment per pocket.  Anchor,
VCG, and RCPT are the immutable matched rows from the
full-76-pocket study; the eight fixed-scalar, ridge, EMA, and norm-cap controls
reuse the same fitted information and deployment seeds.  No decoded molecule
or docking result tunes a controller.  The ten pockets are the equally
weighted experimental units.

\begin{table*}[!htb]
\caption{\textbf{Scaffold simple-baseline comparison.}  Equal-pocket means
over the matched ten-pocket fit-16/deploy-16 subset.  ``Affinity'' is the
fixed-denominator yield that both retains the scaffold and beats the pocket
reference; ``unique useful'' additionally deduplicates molecular identities.
Risk is measured on each method's active deployment path because the optional
common-state audit did not complete; it is therefore a mechanism diagnostic,
not a same-state counterfactual. Bold marks the best value within each of the two controller groups, including ties.}
\label{app:tab:scaffold-simple-baselines}
\centering
\small
\setlength{\tabcolsep}{3.2pt}
\begin{tabular}{@{}lrrrrrrrr@{}}
\toprule
Controller & Risk/Anchor & Valid & Scaffold & Affinity & Unique useful & QED &
SA & Best QVina\\
 & $\downarrow$ & $\uparrow$ & $\uparrow$ & $\uparrow$ & $\uparrow$ &
$\uparrow$ & $\uparrow$ & $\downarrow$\\
\midrule
Anchor & 1.000 & 0.463 & 0.450 & 0.363 & 0.263 & \textbf{0.517} & 0.653 & \textbf{\boldmath $-7.95$}\\
VCG & 13.895 & 0.319 & 0.319 & 0.294 & 0.200 & 0.409 & 0.584 & $-7.66$\\
\rowcolor{oursrow}\ours{} & \textbf{0.890} & \textbf{0.563} & \textbf{0.513} & \textbf{0.413} & \textbf{0.306} & 0.489 & \textbf{0.662} & $-7.93$\\
\midrule
Fixed $\alpha=0.25$ & 1.915 & 0.563 & 0.544 & 0.456 & 0.356 & \textbf{0.548} & 0.663 & $-7.94$\\
Fixed $\alpha=0.50$ & 5.116 & \textbf{0.588} & \textbf{0.569} & \textbf{0.506} & \textbf{0.406} & 0.529 & 0.629 & \textbf{\boldmath $-8.15$}\\
Fixed $\alpha=0.75$ & 11.483 & 0.406 & 0.400 & 0.350 & 0.256 & 0.509 & 0.631 & $-8.10$\\
Ridge $\rho=10^{-3}$ & 18.750 & 0.281 & 0.281 & 0.256 & 0.163 & 0.475 & 0.593 & $-7.87$\\
Ridge $\rho=10^{-2}$ & 18.429 & 0.238 & 0.238 & 0.231 & 0.138 & 0.483 & 0.593 & $-7.75$\\
Ridge $\rho=10^{-1}$ & 20.740 & 0.275 & 0.275 & 0.238 & 0.144 & 0.382 & 0.548 & $-7.69$\\
VCG+EMA $\beta=0.9$ & 2.991 & 0.269 & 0.269 & 0.244 & 0.150 & 0.441 & 0.596 & $-7.53$\\
VCG+norm cap $0.25$ & \textbf{0.994} & 0.544 & 0.538 & 0.456 & 0.356 & 0.507 & \textbf{0.664} & $-8.02$\\
\bottomrule
\end{tabular}
\end{table*}

Table~\ref{app:tab:scaffold-simple-baselines} confirms that RCPT
repairs VCG: validity rises from 0.319 to 0.563, scaffold yield from 0.319 to
0.513, and active-path risk falls from 13.895 to 0.890.  It is not uniformly
best among simple rules on this small subset.  Fixed $\alpha=0.50$ has higher
scaffold, affinity-improved, and unique-useful yields, with paired bootstrap
intervals excluding zero, while the norm cap is numerically strong on several
endpoints.  Conversely, $\alpha=0.75$, every ridge setting, and EMA deteriorate
substantially.  The result supports adaptive negative-transfer protection,
not dominance over an oracle-chosen fixed scalar for every subset.

\subsection{QM9 steering fixed-rule comparison}
\label{app:sec:reward-simple-baselines}

The reward audit uses four seeds, 1,000 integration steps, 16 fitting
particles, and eight 16-particle corrected deployment microbatches (128 slots
per controller and seed).  It uses the normalized three-field class consisting
of the two reward-gradient fields and reverse base drift.  The first deployment
microbatch shares the fitting random stream and the remaining seven are fresh,
so this audit supplements rather than replaces the independently seeded
eight-seed, 12-field primary result in Table~\ref{tab:reward-qm9}.

\begin{table*}[!htb]
\caption{\textbf{QM9 steering simple-baseline comparison.}  Means over four
fit seeds; MAE reports mean $\pm$ seed standard deviation.  Unique success
counts distinct valid molecular identities meeting both targets at the
unchanged one-MAD threshold per 128 generation slots, so duplicates and
invalid samples consume budget.  Risk is the integrated selected/Anchor
residual-variance ratio on each controller's own deployment genealogy and is
not a common-state comparison. Bold marks the best value within each of the two controller groups, including ties.}
\label{app:tab:reward-simple-baselines}
\centering
\small
\setlength{\tabcolsep}{6pt}
\begin{tabular}{@{}lrrr@{}}
\toprule
Controller & Risk/Anchor$\downarrow$ & MAE$\downarrow$ &
Unique success /128$\uparrow$\\
\midrule
Anchor & 1.000 & $0.9346\pm0.0331$ & 20.5\\
VCG & 0.242 & $0.8499\pm0.0205$ & 33.5\\
\rowcolor{oursrow}\ours{} & \textbf{0.224} & \textbf{\boldmath $0.8252\pm0.0462$} & \textbf{39.2}\\
\midrule
Fixed $\alpha=0.25$ & 1.215 & $0.9678\pm0.0996$ & 20.5\\
Fixed $\alpha=0.50$ & 0.330 & $0.8781\pm0.0298$ & 29.2\\
Fixed $\alpha=0.75$ & 0.086 & $0.8593\pm0.0597$ & 35.2\\
Ridge $\rho=10^{-3}$ & 0.157 & $0.8892\pm0.0341$ & 31.8\\
Ridge $\rho=10^{-2}$ & 0.068 & $0.8237\pm0.0438$ & 38.5\\
Ridge $\rho=10^{-1}$ & \textbf{0.041} & \textbf{\boldmath $0.8136\pm0.0638$} & \textbf{39.8}\\
VCG+EMA $\beta=0.9$ & 1.746 & $0.8637\pm0.0639$ & 32.2\\
VCG+norm cap $0.25$ & 0.045 & $0.8412\pm0.0413$ & 33.2\\
\bottomrule
\end{tabular}
\end{table*}

Table~\ref{app:tab:reward-simple-baselines} shows a complementary boundary.
No fixed scalar improves mean MAE over VCG, whereas RCPT lowers
MAE from 0.8499 to 0.8252 and raises unique success from 33.5 to 39.2 while
slightly lowering active-path risk.  Fixed ridge at $\rho=10^{-1}$ leads MAE and unique success,
but its effect is nonmonotone: $\rho=10^{-3}$ is worse than
VCG in both metrics.  With
only four seeds, all paired MAE intervals against VCG include zero; this audit
therefore establishes a fixed-rule sensitivity, not statistically decisive
reward-domain superiority.

\subsection{Bias and penalty-estimator audits}
\label{app:sec:assumption-audits}

\paragraph{Coefficient bias.}
The expected-risk oracle retains
$B=(\lambda^\star)^\top C\,\E[\widehat\lambda-\lambda^\star]$; the structured
model does not estimate it separately.  The Checker audit uses the 12-field basis, $N=8$, 8,192
reference particles, 10,000 independent fitting ensembles, 80 time points,
and the matched ridge convention.  The mean bias-aware and no-bias trust are
0.1900 and 0.1934; their mean absolute pointwise difference is 0.0176 (median
0.0077, 95th percentile 0.0661, maximum 0.1071).  Ignoring $B$ increases the
integrated oracle risk by 0.781\%.  This supports a practical simplification on
Checker, not the general claim $B=0$ or $B=o(N^{-1})$.

\paragraph{Penalty approximation.}
The large-population Checker audit uses five trajectories, 8,192 population
particles per trajectory, 40 fitted times, and 500 repeated fits at
$N\in\{16,32,64,128,256\}$.  Across 200 seed--time problems, the structured
proxy and $N^{-1}\tr(\Omega C^\dagger)$ have Spearman correlation 0.893, but
their integrated scale ratio is only 0.167.  At $N=256$, measured penalty is
$1.071\times$ the $\Omega$-based prediction while the structured proxy
is low by $6.289\times$.  At $N=16$, neither first-order approximation is
quantitatively accurate.

\paragraph{Realized calibration is a separate empirical question.}
Neither preceding expected-penalty audit estimates the realized $\kappa$ required by
Theorem~\ref{thm:structured-calibration}. The corresponding target-law quantities are
$q=\Var_p(a^\top\widehat\lambda)$,
$\kappa=\Cov_p(g^0+a^\top\widehat\lambda,a^\top\widehat\lambda)$, and
$\alpha_{\rm ray}=\Pi_{[0,1]}(1-\kappa/q)$ when $q>0$.
The relevant errors are $S-q$ and $\widehat P-\kappa$, together with
selected-amplitude error, realized ray regret, and negative-transfer
frequency.
When $\kappa<0$, the ray oracle retains the full update, and a nonnegative
penalty has error at least $-\kappa$. Appendix~\ref{app:sec:ray-audit} measures these quantities on an independent corrected-SMC reference. It finds lower selected-ray regret despite substantial proxy error. The reference is approximate and dependent within each particle system, so the audit does not certify the target-law bound for Algorithm~\ref{alg:rcpt}.

\begin{table}[!htb]
\caption{\textbf{Covariance-penalty ablation.} Ten paired Checker seeds.
Replacing the structured effective-rank proxy with a full $\Omega$ estimate
gives no consistent endpoint gain. This changes the penalty estimator,
separately from the Direct LOO amplitude ablation.}
\label{app:tab:loo-omega}
\centering\small
\setlength{\tabcolsep}{7pt}
\begin{tabular}{@{}llrrr@{}}
\toprule
Mode & Fit & Estimator & $W_2\downarrow$ & MMD$\downarrow$\\
\midrule
SMC & 8 & Full $\Omega$ & 2.010 & \textbf{2.917}\\
\rowcolor{oursrow}SMC & 8 & \ours{} & \textbf{1.784} & 3.181\\
\midrule
SMC & 16 & Full $\Omega$ & 2.170 & 2.213\\
\rowcolor{oursrow}SMC & 16 & \ours{} & \textbf{2.030} & \textbf{2.069}\\
\midrule
Proposal & 8 & Full $\Omega$ & 1.769 & 2.519\\
\rowcolor{oursrow}Proposal & 8 & \ours{} & \textbf{1.768} & 2.519\\
\midrule
Proposal & 16 & Full $\Omega$ & 1.469 & 1.275\\
\rowcolor{oursrow}Proposal & 16 & \ours{} & \textbf{1.443} & \textbf{1.269}\\
\bottomrule
\end{tabular}
\end{table}

\section{Additional Mechanism Diagnostics}
\label{app:sec:extended-discussion}

The diagnostics below supplement the main figures with fitting-budget, paired-risk, and genealogy measurements.

\subsection{Mechanism diagnostics}
\label{app:sec:mechanism}

The fitting-population sweep contains ten paired Checker SMC runs at each of
$N_{\rm fit}\in\{4,8,16,32,64,128,256\}$, with 10,000 evaluation particles from a fresh deployment population. All 70 runs per method completed. Figure~\ref{fig:fit-deploy-main}
compares the integrated fitting and independent residual ratios, each with its
own matched Anchor denominator. Low-particle VCG nearly interpolates its
fitting states, although its independent risk can be orders of magnitude
larger. RCPT gives up some fitting improvement to reduce independent risk.
At fit 8 the full sweep has median risks 105.22 and 0.966, and mean risks
524.17 and 1.084 for VCG and RCPT; the mean reflects a remaining RCPT tail.
This fitting-population sweep is separate from the fixed-schedule endpoint comparison in Table~\ref{tab:six-baseline-2d}.

Figure~\ref{app:fig:information} displays the same sweep against particle count
on both panels. The median support ratios $N_{\rm fit}/r_{\rm eff}$ for RCPT
are 1.33, 1.15, 1.48, 2.80, 5.50, 10.88, and 21.75; the effective rank itself
changes at small $N$. Mean amplitude rises from 0.143 at $N=4$ to 0.921 at
$N=256$, while VCG and RCPT independent risks converge. The method therefore retains more of the fitted update as the independent risks converge; this observation does not identify the population-optimal headroom.

\begin{figure}[!htb]
\centering
\includegraphics[width=0.7\textwidth]{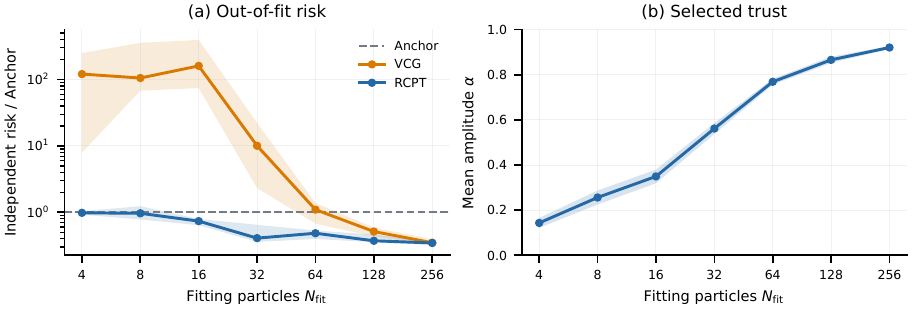}
\caption{\textbf{Increasing fitting information closes the gap to VCG.}
Risk curves show medians and interquartile ranges; amplitude shows mean
and standard error over ten seeds at each particle count.}
\label{app:fig:information}
\end{figure}

Figure~\ref{app:fig:repair-retention} compares independent risk across the
paired methods. RCPT improves 66/70 Checker runs and repairs 34/47 harmful
VCG runs to Anchor or below. The separate Scaffold common-state audit contains
75 pockets; pocket 45 was reserved for the engineering smoke test before this
audit. RCPT improves every pocket and repairs 68/73 harmful VCG schedules.
This diagnostic cohort is distinct from the full 76-pocket deployment table.
The ratios are not pooled across domains.

\begin{figure}[!htb]
\centering
\includegraphics[width=0.6\textwidth]{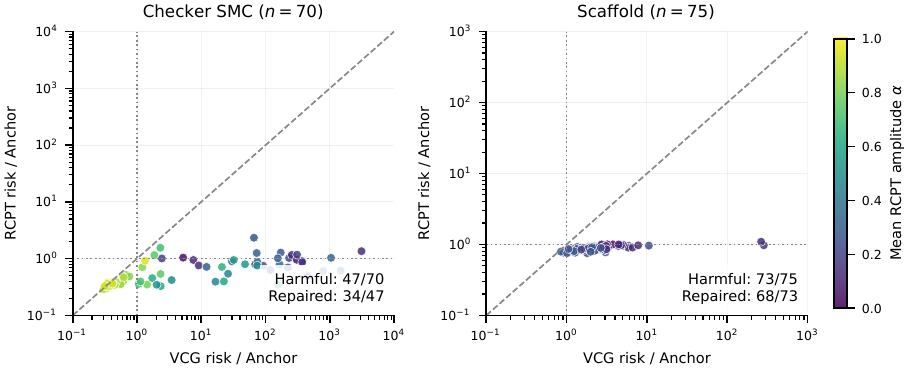}
\caption{\textbf{Risk repair and retained adaptation.} Points below the diagonal
favor RCPT. Points to the right of the vertical Anchor line and below the
horizontal one are harmful VCG fits repaired to Anchor or below. Color denotes
mean selected amplitude, not terminal quality.}
\label{app:fig:repair-retention}
\end{figure}

\begin{figure}[!htb]
\centering
\includegraphics[width=0.6\textwidth]{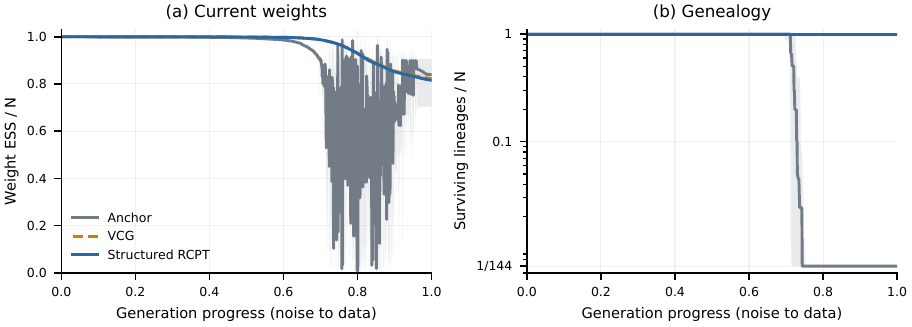}
\caption{\textbf{Weight ESS and genealogy in online QM9 steering.} Median and
interquartile range across eight $N=144$ seeds, ordered in generation time.
Anchor resamples 85.8 times on average and ends with one ancestor in every
seed; VCG and RCPT never resample and retain all 144. Resampling can restore
weight ESS without restoring lost ancestry. These are observed path-dependent
outcomes, not a universal consequence of local residual ordering.}
\label{app:fig:reward-genealogy}
\end{figure}

\subsection{Selected amplitude across tasks}
\label{app:sec:alpha}

Figure~\ref{fig:alpha-tasks} summarizes time-resolved amplitudes. Thick lines are means across
units, bands are interquartile ranges, and thin lines are individual schedules.
The summary averages the saved time nodes without claiming a common
continuous-time integral. Native time runs in opposite generation directions
in some domains. Amplitude is relative to each fitted coefficient; it is not
a comparable drift norm or an independent measurement of population headroom.

Figure~\ref{fig:alpha-tasks} contains 76 Scaffold pockets, eight online QM9
seeds with 144 live particles, ten Checker schedules, and ten CIFAR schedules.
The updated Checker panel shows strong early retention followed by smaller
amplitudes, while the CIFAR panel is mostly conservative. The CIFAR panel displays ten schedules; the replicated risk audit
aggregates 50 class--seed cells. These curves describe the rule's response within each plotted
cohort; they do not identify a terminal-optimal amplitude.

\subsection{Residual risk as a terminal surrogate}

Across 83 finite Proposal-Only profiles, integrated held-out residual has
Spearman correlation $0.775$ with empirical TV (Pearson $0.881$ after log
scaling its many-order-of-magnitude range), while the theorem-aligned
$B_1=\tfrac12\int_0^1\sqrt{V_t}\dd t$ has Spearman correlation $0.730$.
Retaining the time-resolved profile matters: $B_1$ is a median $47\%$ tighter
than the additional time--Cauchy relaxation
$B_2=\tfrac12\sqrt{\int_0^1V_t\dd t}$.  Nevertheless, every evaluated $B_1$
exceeds one and its median is $5.99\times$ empirical TV.  We therefore use
residual risk as a correlated comparative diagnostic, not an exact terminal
error estimate.

\subsection{Cross-domain portability and fixed controls}

The CIFAR-10 study in Appendix~\ref{app:sec:image-dlsm} extends the
finite-particle mechanism to a pretrained image model. Its independent-state risk audits and separate terminal cohorts answer
different questions: residual generalization and downstream behavior. The
current terminal comparison uses one seed with matched target power and Anchor guidance 2.5.

Appendix Tables~\ref{app:tab:scaffold-simple-baselines} and
\ref{app:tab:reward-simple-baselines} compare fixed shrinkage, ridge, and
engineering safeguards.  The best fixed heuristic does not transfer: a fixed
scalar can be strong on selected scaffold yields, while no fixed scalar
improves QM9 steering MAE over VCG; ridge behavior is similarly task-dependent
and nonmonotone.  RCPT is the common arm that estimates trust from the
available fit rather than selecting a task-specific regularization value.

\section{Computational Cost}
\label{app:sec:cost}

After expert fields and divergences are available, VCG costs
$O(Nm^2+m^3)$ per fitted step. Leave-one-out refitting reuses those features and adds
small dense linear algebra, but no expert, denoiser, or evaluator call.
Table~\ref{tab:cost-main} reports matched costs with hardware; all times
exclude downstream docking and offline evaluation.

Scaffold timing means cover all 76 pockets on an RTX A6000. QM9 times use four A100-SXM4 runs; runs on other hardware are excluded. The 1,000-step QM9 runs predated memory instrumentation, so peak memory comes from a matched 10-step probe. Small VCG--RCPT timing reversals reflect run variability. RCPT adds 3.13 seconds over VCG in approximately 3,880 seconds (0.081\%) with equal probed memory; the recorded Anchor time is 3877.460 seconds.

\paragraph{Why the small-particle regime matters computationally.}
RCPT's main systems benefit is not that leave-one-out refitting makes one trajectory
cheaper; it is that a controller can be fit on the small ensemble the user
actually wants to generate. Table~\ref{app:tab:cost-particles} measures this
scaling in our implementations.

\begin{table}[!htb]
\caption{Particle scaling for the 12-field scaffold RCPT fit on NVIDIA RTX
A6000 48GB GPUs, averaged over the complete 76-pocket benchmark. The \textit{Solve}
column is aggregate proposal-controller solve time; the remainder is dominated
by trajectory and field/model computation.}
\label{app:tab:cost-particles}
\centering
\small
\begin{tabular}{@{}rrrr@{}}
\toprule
$N_{\rm fit}$ & s/pocket & Peak GB & Solve s\\
\midrule
16 & 197.445 & 4.740 & 5.282\\
32 & 268.852 & 9.365 & 5.473\\
64 & 419.193 & 18.611 & 5.478\\
\bottomrule
\end{tabular}
\end{table}

From 16 to 64 scaffold particles, allocated VRAM rises by $3.93\times$ and
wall time by $2.12\times$, whereas aggregate controller-solve time changes by
only 0.20 seconds. On the same A100, matched QM9 steering probes require 4.569 GB
at $N=16$ and 40.204 GB at $N=144$ ($8.80\times$ memory for $9\times$ as many
live particles). These are direct measurements of particle scaling, not an
extrapolation to thousands of particles.

These matched measurements support calibration at low particle budgets:
12-field controllers fitted from 8--16 particles add negligible computation
over VCG, while larger fitting populations increase model cost and memory.
They do not measure a speedup over DriftLite's published high-particle
experiments~\citep{ren2026driftlite}.

A frozen schedule also amortizes its fitting time over independent outputs.
Deployment remains microbatchable: it retains the fields needed to apply the
combination while removing online covariance, LOO, and particle-synchronization
work.

\end{document}